\documentclass[11pt,a4paper,final]{article}
\pdfoutput=1
\usepackage[left=1in,right=1in,top=1in,bottom=1in]{geometry}

\usepackage[T1]{fontenc}
\usepackage[utf8]{inputenc}
\usepackage{amsmath}
\usepackage{amsfonts}
\usepackage{amssymb}
\usepackage{graphicx}
\usepackage{xcolor}

\usepackage{booktabs} 
\usepackage[ruled]{algorithm2e} 

\usepackage{tikz}
\usetikzlibrary{arrows}
\usetikzlibrary{shapes.multipart}
\usepackage{caption}
\usepackage{natbib}

\usepackage[T1]{fontenc}    
\usepackage{hyperref}       
\usepackage{url}            
\usepackage{booktabs}       
\usepackage{nicefrac}       
\usepackage{microtype}      
\usepackage{mathtools}
\usepackage{amsthm}
\usepackage{bbm}
\usepackage{wrapfig}
\usepackage{floatrow}
\usepackage{enumitem}
\usepackage{authblk}
\usepackage{cleveref}
\usepackage{comment}

\usepackage{color-edits}
\addauthor{kh}{red}
\addauthor{sw}{blue}
\addauthor{jw}{purple}
\addauthor{aa}{orange}

\usepackage[noend]{algpseudocode}

\usepackage{caption}
\usepackage{subcaption}

\newtheorem{theorem}{Theorem}[section]
\newtheorem{lemma}[theorem]{Lemma}
\newtheorem{fact}[theorem]{Fact}

\newtheorem{corollary}[theorem]{Corollary}

\newtheorem{assumption}[theorem]{Assumption}

\newtheorem{definition}[theorem]{Definition}

\newcommand{\cN}{\mathcal{N}}

\newcommand{\cV}{\mathcal{V}}

\newcommand{\cA}{\mathcal{A}}

\newcommand{\cI}{\mathcal{I}}

\newcommand{\vY}{\mathbf{Y}}
\newcommand{\vX}{\mathbf{X}}
\newcommand{\vZ}{\mathbf{Z}}
\newcommand{\vA}{\mathbf{A}}
\newcommand{\vU}{\mathbf{U}}
\newcommand{\vV}{\mathbf{V}}
\newcommand{\vP}{\mathbf{P}}

\newcommand{\vB}{\mathbf{B}}

\newcommand{\rank}{\mathrm{rank}}
\renewcommand{\S}{\mathbb{S}}
\newcommand{\err}{\mathrm{err}}
\NewDocumentCommand{\pv}{m e{_} m}{%
  #1\IfValueT{#2}{_{#2}}^{(#3)}%
}
\newcommand{\snr}{\mathrm{snr}}
\newcommand{\E}{\mathbb{E}}
\newcommand{\R}{\mathbb{R}}

\newcommand{\calF}{\mathcal{F}}

\newcommand{\diag}{\mathrm{diag}}
\newcommand{\spn}{\mathrm{span}}
\newcommand{\wh}[1]{\widehat{#1}}
\newcommand{\wt}[1]{\widetilde{#1}}

\newcommand{\ul}[1]{\underline{#1}}

\DeclareFontFamily{U}{mathx}{\hyphenchar\font45}
\DeclareFontShape{U}{mathx}{m}{n}{
      <5> <6> <7> <8> <9> <10>
      <10.95> <12> <14.4> <17.28> <20.74> <24.88>
      mathx10
      }{}
\DeclareSymbolFont{mathx}{U}{mathx}{m}{n}
\DeclareFontSubstitution{U}{mathx}{m}{n}
\DeclareMathAccent{\widecheck}{0}{mathx}{"71}
\newcommand{\wc}[1]{\widecheck{#1}}
\newcommand{\Epsilon}{\mathcal{E}}

\begin{document}
\title{Adaptive Principal Component Regression\\ with Applications to Panel Data\footnote{A version of this paper was pubished in NeurIPS 2023 (Advances in Neural Information Processing Systems 37).\looseness-1}}

\author[1]{Anish Agarwal\thanks{For part of this work, Anish Agarwal was a postdoc at Amazon, Core AI.}}
\author[2]{Keegan Harris}
\author[2]{Justin Whitehouse}
\author[2]{Zhiwei Steven Wu}

\affil[1]{Columbia University}
\affil[2]{Carnegie Mellon University}
\affil[ ]{\texttt{aa5194@columbia.edu}}
\affil[ ]{\texttt {\{keeganh,jwhiteho,zhiweiw\}@cs.cmu.edu}}

\date{}

\maketitle

\pagenumbering{gobble}
\begin{abstract}
Principal component regression (PCR) is a popular technique for fixed-design \emph{error-in-variables} regression, a generalization of the linear regression setting in which the observed covariates are corrupted with random noise.
We provide the first time-uniform finite sample guarantees for (regularized) PCR whenever data is collected \emph{adaptively}.
Since the proof techniques for analyzing PCR in the fixed design setting do not readily extend to the online setting, our results rely on adapting tools from modern martingale concentration to the error-in-variables setting.
We demonstrate the usefulness of our bounds by applying them to the domain of \emph{panel data}, a ubiquitous setting in econometrics and statistics. 
As our first application, we provide a framework for experiment design in panel data settings when interventions are assigned adaptively.
Our framework may be thought of as a generalization of the \emph{synthetic control} and \emph{synthetic interventions} frameworks, where data is collected via an adaptive intervention assignment policy.
Our second application is a procedure for \emph{learning} such an intervention assignment policy in a setting where units arrive sequentially to be treated. 
In addition to providing theoretical performance guarantees (as measured by \emph{regret}), we show that our method empirically outperforms a baseline which does not leverage error-in-variables regression. 
\end{abstract}

\newpage
\tableofcontents
\newpage

\pagenumbering{arabic}
\section{Introduction}
An omnipresent task in machine learning, statistics, and econometrics is that of making predictions about outcomes of interest given an action and  conditioned on observable covariates.
An often overlooked aspect of the prediction task is that in many settings the learner only has access to \emph{imperfect observations} of the covariates, due to e.g. measurement error or inherent randomness in the problem domain. 
Such settings are sometimes formulated as \emph{error-in-variables} regression: a \emph{learner} is given access to a collection of data $(Z_n, a_n, Y_n)_{n \geq 1}$, where $Z_n \in \R^d$ are the \emph{observed covariates}, $a_n \in \{1, \ldots, A\}$ is the \emph{action taken}, and $Y_n \in \R$ is the \emph{outcome} for each observation $n$. 
Typically, the outcomes are assumed to be generated by a linear model $Y_n := \langle \theta(a_n), X_n \rangle + \xi_n$ and $Z_n := X_n + \epsilon_n$, where $X_n \in \R^d$ are the \emph{true covariates}, $\epsilon_n \in \R^d$ is the covariate noise, $\theta(a_n) \in \R^d$ is an \emph{unknown slope vector} associated with action $a_n$, and $\xi_n \in \R$ is the response noise. 
Note that the learner does not get to see the true covariates $X_n$.
Observe that when $\epsilon_n = 0$ we recover the traditional linear regression setting.
Such an error-in-variables model can encompass many forms of data corruption including measurement error, missing values, discretization, and differential privacy---see~\cite{agarwal2020robustness, agarwal2021causal} for details.
\looseness-1

Our point of departure from previous work is that we allow the sequence of data $(Z_n, a_n, Y_n)_{n \geq 1}$ to be chosen \emph{adaptively}. 
In other words, we provide bounds for learning in the error-in-variables regression setting when the data seen at the current round $n$ is allowed to depend on the previously-seen data $(Z_{m}, a_m, Y_m)_{1 \leq m<n}$. 
Adaptive data collection occurs when the choices of future observations can depend on the inference from previous data, which is common in learning paradigms such as multi-armed bandits~\cite{lattimore2020bandit, slivkins2019introduction}, active learning~\cite{settles2009active}, and time-series analysis~\cite{shumway2000time, enders2008applied}. Similar to prior work on adaptive data collection that shows that valid statistical inference can be done when the true covariates are observed~\cite{deshpande2018accurate, nie2018adaptively, hadad2021confidence, zhang2021statistical, zhang2022statistical}, our work provides the first time-uniform finite sample guarantees for error-in-variables regression using adaptively collected data.\looseness-1

Concretely, we focus on analyzing \emph{principal component regression (PCR)}~\cite{jolliffe1982note}, a method that has been shown to be effective for learning from noisy covariates~\cite{agarwal2020robustness, agarwal2020principal} and a central tool for learning from panel data \cite{agarwal2020synthetic, agarwal2021causal, agarwal2020principal}. 
At a high level, PCR ``de-noises'' the sequence of observed covariates $(Z_n)_{n \geq 1}$ as $(\wh{Z}_n)_{n \geq 1}$ by performing hard singular value thresholding, after which a linear model is learned using the observed outcomes $(Y_n)_{n \geq 1}$ and the denoised covariates $(\wh{Z}_n)_{n \geq 1}$.
See~\Cref{sec:pcr} for further technical background on PCR.
%
%
\subsection{Contributions}
\begin{enumerate}
    \item We derive novel time-uniform bounds for an online variant of regularized PCR when the sequence of covariates is chosen adaptively, i.e. the data seen at the current round is allowed to depend on previously-seen data.
    The techniques used to derive bounds for PCR in the fixed-sample regime~\cite{agarwal2020principal} do not extend to the setting in which data is collected adaptively; thus, we require new tools and ideas to obtain our results.
    Specifically, our results rely on applying recent advances in martingale concentration~\cite{howard2020time, howard2021time}, as well as more classical results on self-normalized concentration~\cite{abbasi2011improved, de2004self, de2007pseudo} which are commonly applied to online regression problems, to the error-in-variables setting.
    As an example of the bounds we obtain, consider the task of estimating the underlying relationship $\theta(a)$ between true (i.e. noiseless) covariates and observations, given access to $n$ adaptively-chosen noisy covariates and their corresponding actions and observations. 
    The $\ell_2$ estimation error of the adaptive PCR estimator $\wh{\theta}_n(a)$ can be bounded as
    \begin{equation*}
        \|\wh{\theta}_n(a) - \theta(a)\|_2^2 =\wt{O}\left(\frac{1}{\snr_n(a)^2}\kappa(\vX_n(a))^2\right)
    \end{equation*}
    with high probability, where $\snr_n(a)$ is the \emph{signal-to-noise ratio} associated with action $a$ at round $n$ (\Cref{def:snr}), a measure of how well the true covariates stand out from the noise.
    $\kappa(\vX_n(a))$ is the \emph{condition number} of the true covariates. 
    Intuitively, if $\snr_n(a)$ is high the true covariates can be well-separated from the noise, and therefore PCR accurately estimates $\theta(a)$ as long as the true covariates are well-conditioned.\looseness-1

    Despite the harder setting we consider, our PCR bounds for adaptively-collected data largely match the bounds currently known for the fixed-sample regime, and even improve upon them in two important ways: 
    (1) Our bounds are \emph{computable}, i.e. they depend on \emph{known} constants and quantities available to the algorithm. 
    (2) Unlike~\citet{agarwal2020principal}, our bounds do not depend on the $\ell_1$-norm of $\theta(a)$, i.e., we do not require approximate sparsity of $\theta(a)$ for the bounds to imply consistency.
    This is important because PCR is a rotationally-invariant algorithm, and so its performance guarantees should not depend on the orientation of the basis representation of the space to be learned. 
    The price we pay for adaptivity is that $\snr_n (a)$ is defined with respect to a bound on the \emph{total} amount of noise seen by the algorithm so far, instead of just the noise associated with the rounds that $a$ is taken. 
    As a result, our bound for $\wh{\theta}_n(a)$ may not be tight if $a$ is seldomly selected.
    \item We apply our PCR results to the problem of online experiment design with \emph{panel data}.
    In panel data settings, the learner observes repeated, noisy measurements of \emph{units} (e.g. medical patients, subpopulations, geographic locations) under different \emph{interventions} (e.g. medical treatments, discounts, socioeconomic policies) over \emph{time}. 
    This is an ubiquitous method of data collection, and as a result, learning from panel data has been the subject of significant interest in the econometrics and statistics communities (see~\Cref{sec:related}).

    A popular framework for learning from panel data is \emph{synthetic control} (SC)~\cite{abadie2003economic, abadie2010synthetic}, which uses historical panel data to estimate counterfactual unit measurements under control. 
    Synthetic interventions (SI)~\cite{agarwal2020synthetic} is a recent generalization of the SC framework which allows for counterfactual estimation under treatment, in addition to control.
    By leveraging online PCR, we can perform counterfactual estimation of unit-specific treatment effects under both treatment and control, as in the SI framework.
    However, unlike the traditional SI framework, we are the first to establish statistical rates for counterfactual unit outcome estimates under different interventions \emph{while allowing for both units and interventions to be chosen adaptively}. 
    Such adaptivity may naturally occur when treatments are prescribed to new units based on the outcomes of previous units. 
    For example, this is the case when the intervention chosen for each unit is the one which appears to be best based on observations in the past.\looseness-1

    As a second application, we also leverage our bounds to obtain provable performance guarantees for an algorithm (\Cref{alg:etc}) which learns how to assign interventions to units in order to optimize some objective (e.g., maximizing engagement, minimizing costs). 
    Here adaptive data collection occurs naturally, as the intervention~\Cref{alg:etc} assigns to the current unit depends on the outcomes of previously-seen units. 
    Using simulations, we show that~\Cref{alg:etc} compares favorably to methods which do not leverage error-in-variables regression techniques. 
\end{enumerate}
\subsection{Related Work}\label{sec:related}
\paragraph{Error-in-variables regression} 
There is a rich literature on error-in-variables regression (e.g. \cite{griliches1970error, kim1990robust, chesher1991effect, hall2008measurement, wansbeek2001measurement, hausman2001mismeasured, fuller2009measurement}), with research focusing on topics such as high-dimensional~\cite{loh2011high, kaul2015weighted, datta2017cocolasso, rosenbaum2010sparse} and Bayesian settings~\cite{reilly1981bayesian, ungarala2000multiscale, figueroa2022robust}.
Principal component regression (PCR)~\cite{jolliffe1982note, bair2006prediction, agarwal2020robustness, agarwal2020principal} is a popular method for error-in-variables regression. 
The results of~\citet{agarwal2020principal} are of particular relevance to us, as they provide finite sample guarantees for the fixed design (i.e. non-adaptive) version of the setting we consider.

\paragraph{Self-normalized concentration} There has been a recent uptick in the application of self-normalized, martingale concentration to online learning problems. In short, self-normalized concentration aims to control the growth of processes that have been normalized by a random, or empirical, measure of accumulated variance \citep{de2004self, de2007pseudo, howard2020time, howard2021time, whitehouse2023self}. Self-normalized concentration has led to breakthroughs in wide-ranging areas of machine learning such as differential privacy \citep{whitehouse2022brownian, whitehouse2022fully}, PAC-Bayesian learning \citep{chugg2023unified}, convex divergence estimation \citep{manole2023martingale}, and online learning \citep{whitehouse2023improved, chowdhury2017kernelized, abbasi2011improved}. Of particular importance for our work are the results of \citet{abbasi2011improved}, which leverage self-normalized concentration results for vector-valued processes \citep{de2004self, de2007pseudo} to construct confidence ellipsoids for online regression tasks. We take inspiration from these results when constructing our estimation error bounds for PCR in the sequel.

\paragraph{Learning in panel data settings}
Our application to panel data builds off of the SI framework \cite{agarwal2020synthetic, causal_MC}, which itself is a generalization of the canonical SC framework for learning from panel data \cite{abadie2003economic, abadie2010synthetic, Hsiao12, imbens16, athey1, LiBell17, xu_2017, rsc, mrsc, Li18, ark, bai2020matrix, asc, chernozhukov2020practical, fernandezval2020lowrank}.
In both frameworks, a \emph{latent factor model} is often used to encode structure between units and time-steps \cite{chamberlain, liang_zeger, arellano, bai03, bai09, pesaran, moon_15, moon_weidner_2017}.
Specifically, it is assumed that unit outcomes are the product of unit- and time/intervention-specific latent factors, which capture the heterogeneity across time-steps, units, and interventions, and allows for the estimation of unit-specific counterfactuals under treatment and control.
Other extensions of the SI framework include applications in biology~\cite{squires2022causal}, network effects~\cite{agarwal2022network}, combinatorially-many interventions~\cite{agarwal2023synthetic}, and intervening under incentives~\cite{harris2022strategyproof, ngo2023incentivized}.

Finally, there is a growing line of work at the intersection of online learning and panel data.
\citet{chen2023synthetic} views the problem of SC as an instance of online linear regression, which allows them to apply the regret guarantees of the online learning algorithm \emph{follow-the-leader} \cite{kalai2005efficient} to show that the predictions of SC are comparable to those of the best-in-hindsight weighted average of control unit outcomes.
\citet{farias2022synthetically} build on the SC framework to estimate treatment effects in adaptive experimental design, while minimizing the regret associated with experimentation.
The results of~\citet{farias2022synthetically} are part of a growing line of work on counterfactual estimation and experimental design using multi-armed bandits~\cite{qin2022adaptivity, simchi2023multi, zhang2022causal, carranza2023flexible}.\looseness-1

\section{Setting and Background}\label{sec:setting}

\paragraph{Notation} We use boldface symbols to represent matrices. 
For $N \in \mathbb{N}$, we use the shorthand $[N] := \{1, \ldots, N\}$. 
Unless specified otherwise, $\|v\|$ denotes the $\ell_2$-norm of a vector $v$, and $\|\mathbf{A}\|_{op}$ the operator norm of matrix $\mathbf{A}$.
We use $\diag(a_1, \ldots, a_k)$ to represent a $k \times k$ diagonal matrix with entries $a_1, \ldots, a_k$.
For two numbers $a,b \in \mathbb{R}$, we use $a \land b$ as shorthand for $\min\{a,b\}$, and $a \vee b$ to mean $\max\{a,b\}$.
Finally, $\S^{d- 1}$ denotes the $d$-dimensional unit sphere.

\subsection{Problem Setup}
We now describe our error-in-variables setting. 
We consider a \emph{learner} who interacts with an \emph{environment} over a sequence of rounds.
At the start of each round $n \geq 1$, the environment generates covariates $X_n \in W^\ast \subset \mathbb{R}^d$, where $W^\ast$ is a low-dimensional linear subspace of dimension $\dim(W^\ast) = r < d$. 
We assume that $r$ (but not $W^\ast$) is known to the learner. 
Such ``low rank'' assumptions are reasonable whenever, e.g. data is generated according to a \emph{latent factor model}, a popular assumption in high-dimensional statistical settings~\cite{jenatton2012latent, agarwal2009regression, hoff2009multiplicative}.
As we will see in~\Cref{sec:panel}, analogous assumptions are often also made in panel data settings. 
The learner then observes \emph{noisy} covariates $Z_n := X_n + \epsilon_n$, where $\epsilon_n \in \R^d$ is a random noise vector.
Given $Z_n$, the learner selects an \emph{action} $a_n \in [A]$ and observes $Y_n := \langle \theta(a_n), X_n \rangle + \xi_n$, where $\xi_n$ is random noise and $\theta(a)$ for $a \in [A]$ are unknown slope vectors in $W^*$ that parameterize action choices such that $\|\theta(a)\|_2 \leq L$ for some $L \in \mathbb{R}$.
We require that the covariate noise is ``well-behaved'' according to one of the two following assumptions:
\begin{assumption}[\textbf{SubGaussian Covariate Noise}]
\label{ass:noise1}
For any $n \geq 1$, the noise variable $\epsilon_n$ satisfies (a) $\epsilon_n$ is $\sigma$-subGaussian, (b) $\E\epsilon_n = 0$, and (c) $\|\E \epsilon_n\epsilon_n^\top\|_{op} \leq \gamma$, for some constant $\gamma > 0$.

\end{assumption}
\begin{assumption}[\textbf{Bounded Covariate Noise}]
\label{ass:noise2}
For any $n \geq 1$, the noise variable $\epsilon_n$ satisfies (a) $\|\epsilon_n\| \leq \sqrt{Cd}$, (b) $\E\epsilon_n = 0$, and (c) $\E \epsilon_n\epsilon_n^\top = \Sigma$, for some positive-definite matrix $\Sigma$ satisfying $\|\Sigma\|_{op} \leq \gamma$, for some constant $\gamma > 0$.

\end{assumption}
Note that~\Cref{ass:noise2} is a special case of~\Cref{ass:noise1}, which allows us to get stronger results in some settings.
We also impose the following constraint on the noise in the outcomes.\looseness-1
\begin{assumption}[\textbf{SubGaussian Outcome Noise}]
\label{ass:noise_resp}
For any $n \geq 1$, the noise variable $\xi_n$ satisfies (a) $\E\xi_n = 0$, (b) $\xi_n$ is $\eta$-subGaussian, and (c) $\E \xi_n^2 \leq \alpha$, for some constant $\alpha$.
\end{assumption}
Under this setting, the goal of the learner is to estimate $\theta(a)$ for $a \in [A]$ given an (possibly adaptively-chosen) observed sequence $(Z_n, a_n, Y_n)_{n \geq 1}$.
For $n \geq 1$, we define the matrix $\vZ_n \in \R^{n \times d}$ to be the matrix of \emph{observed} (i.e. noisy) covariates, with $Z_1, \dots, Z_n$ as its rows. 
Similarly, $\vX_n = (X_1, \dots, X_n)^{\top} \in \R^{n \times d}$ is the matrix of \emph{noiseless} covariates (which are unobserved), and $\Epsilon_n = (\epsilon_1, \dots, \epsilon_n)^{\top} \in R^{n \times d}, \vY_n = (Y_1, \ldots, Y_n)^{\top} \in \R^{n \times 1},$ and $\Xi_n = (\xi_1, \dots, \xi_n)^{\top} \in \R^{n \times 1}$ are defined analogously.
For any action $a \in [A],$ let $N_n(a) := \{s \leq n : a_s = n \}$ be the \emph{set of rounds} up to and including round $n$ on which action $a$ was chosen. 
Likewise, let $c_n(a) := |N_n(a)|$ denote the \emph{number of rounds} by round $n$ on which action $a$ was chosen. 
For $a \in [A]$, we enumerate $N_n(a)$ in increasing order as $i_1 \leq \cdots \leq i_{c_n(a)}$. 
Finally, we define $\vZ_n(a) \in \R^{c_n(a) \times d}$ to be $\vZ(a) = (Z_{i_1}, \dots, Z_{i_{c_n(a)}})^{\top}$, and define $\vX_n(a), \Epsilon_n(a), \vY_n(a),$ and $\Xi_n(a)$ analogously.\looseness-1

\subsection{Principal Component Regression}\label{sec:pcr}
\paragraph{Background on singular value decomposition} Any matrix $\vA \in \R^{n \times d}$ may be written in terms of its singular value decomposition $\vA = \vU \Sigma \vV^{\top}$, where $\vU \in \R^{n \times d\land n}$ and $\vV \in \R^{d \times d \land n}$ are matrices with orthonormal columns, and $\Sigma = \diag(\sigma_1(\vA), \dots, \sigma_{d \land n}(\vA)) \in \R^{(d \land n) \times (d \land n)}$ is a diagonal matrix containing the singular values of $\vA$, where we assume $\sigma_1(\vA) \geq \dots \geq \sigma_{d \land n}(\vA) \geq 0$. 
Given a \emph{truncation level} $k$, we define the truncation of $\vA$ onto its top $k$ principal components as $\vA_k := \vU_{k} \diag(\sigma_1(\vA), \dots, \sigma_{k \land d \land n}(\vA))\vV^{\top}_k$, where $\vU_k \in \R^{n \times k \land d \land n}$ is the matrix with the first $k \land d \land n$ columns of $\vU$, and $\vV_k \in \R^{n \times k \land d \land n}$ is defined analogously. 
Given such a singular value decomposition, we can define the projection matrix onto the subspace spanned by the top $k$ right singular vectors as $\vP_{k} \in \R^{d \times d }$ given by $\vP_k := \vV_{k}\vV_k^{\top}$.

For $n \geq 1$, $a \in [A]$, and $\vZ_n(a)$ as defined above, we write the $k$-truncated singular value decomposition of $\vZ_n(a)$ as $\vZ_{n, k}(a) = \wh{\vU}_{n, k}(a)\diag(\sigma_1(\vZ_n(a)), \dots, \sigma_{k \land n \land d}(\vZ_n(a)))\wh{\vV}^{\top}_{n, k}(a)$, and the corresponding projection onto the top $k$ right singular vectors of $\vZ_n(a)$ as $\wh{\vP}_{n,k}(a)$. 
When $k = r$, we leverage the simplified notation $\wh{\vP}_n(a) := \wh{\vP}_{n, r}(a)$. 
(Recall $r = \dim(W^*)$.)
By $\vP$, we denote the projection matrix onto the true, underlying subspace $W^\ast$. While $\vP$ is never known, our results leverage the fact that $\wh{\vP}_n(a)$ converges to $\vP$ nicely over time. 
We define the projected noisy covariate matrix matrix to be $\wh{\vZ}_n(a) :=  \vZ_n(a)\wh{\vP}_n(a)$, and define $\wh{\vX}_n(a), \wh{\Epsilon}_n(a)$ similarly. 
Any quantity with a ``\;$\wc{\cdot}$\;'' is defined equivalently to quantities with ``\;$\wh{\cdot}$\;'', except with $\vP$ in place of $\wh{\vP}_n(a)$.
We are now ready to introduce our procedure for estimating $\theta(a)$ for $a \in [A]$, called \emph{adaptive} (or \emph{online}) principal component regression.\looseness-1
\begin{definition}[\textbf{Adaptive Principal Component Regression}]\label{def:pcr}
Given \emph{regularization parameter} $\rho \geq 0$ and \emph{truncation level} $k \in \mathbb{N}$, for $a \in [A]$ and $n \geq 1$ let $\wh{\vZ}_n(a) := \vZ_n(a)\wh{\vP}_{n,k}(a)$ and $\wh{\cV}_n(a) := \wh{\vZ}_n(a)^{\top}\wh{\vZ}_n(a) + \rho \wh{\vP}_{n,k}(a)$. 
Regularized principal component regression estimates $\theta(a)$ as 
\begin{equation*}
    \wh{\theta}_n(a) := \wh{\cV}_n(a)^{-1}\wh{\vZ}_n(a)\vY_n(a).
\end{equation*}
\end{definition}
Setting $\rho = 0$ recovers the version of PCR used in~\citet{agarwal2020principal}. 
In words, (unregularized) PCR ``denoises'' the observed covariates by projecting them onto the subspace given by their $k$-truncation, before estimating $\theta(a)$ via linear regression using the projected covariates. 
We choose to regularize since it is known that regularization increases the stability of regression-style algorithms. 
This added stability from regularization allows us to exploit the online regression bounds of~\citet{abbasi2011improved} to measure the performance of our estimates. 
Throughout the sequel, we only consider adaptive PCR with truncation level $k=r$. 

\subsection{Signal to Noise Ratio}
We now introduce the concept of \textit{signal to noise ratio}, which will be integral in stating and proving our results. 
The signal to noise ratio provides a measure of how strongly the true covariates (this is the ``signal'' of the problem, measured through $\sigma_r(\vX_n(a))$) stand out sequentially with respect to the ``noise'' induced by $\Epsilon_n$ (which we will measure through the relevant high probability bounds on $\|\Epsilon_n\|_{op}$).\looseness-1
\begin{definition}[\textbf{Signal to Noise Ratio}]\label{def:snr}
    We define the signal to noise ratio associated with an action $a \in [A]$ at round $n$ as
    \begin{equation*}
        \snr_n(a) := \frac{\sigma_r(\vX_n(a))}{U_n},
    \end{equation*}
    where $(U_n)_{n \geq 1}$ is a noise-dependent sequence growing as $U_n =O\left(\sqrt{n} + \sqrt{d} + \sqrt{\log\left(\frac{1}{\delta}\log(n)\right)}\right)$. 
\end{definition}
The price we pay for adaptivity is encoded directly into the definition of the signal to noise ratio, $\snr_n(a)$. While one may imagine defining $\snr_n(a)$ as the ratio between $\sigma_r(\vX_n(a))$ and $\|\Epsilon_n(a)\|_{op}$, bounding $\|\Epsilon_n(a)\|_{op}$ is a nontrivial task as the rows of $\Epsilon_n(a)$ may be strongly correlated. To circumvent this, we apply the trivial bound $\|\Epsilon_n(a)\|_{op} \leq \|\Epsilon_n\|_{op}$. Thus, \emph{the price of adaptivity in our setting is that the signal from covariates associated with an action $a$ must stand out with respect to the \ul{total} covariate noise by time $n$}.
The growth condition on $U_n$ presented in Definition~\ref{def:snr} is motivated as follows: w.h.p $\E\left\|\Epsilon_n\right\|_{op} \approx \sqrt{d} + \sqrt{n}$, and the extra additive $\sqrt{\log\left(\frac{1}{\delta}\log(n)\right)}$ factor is the price we pay for having high probability control of $\|\Epsilon_n\|_{op}$ uniformly over rounds. 
Below we provide an exact definition for $U_n$, as this choice leads to bounds with known constants and simple conditions on $\snr_n(a)$ for validity.

We consider the following two sequences $(U_n)_{n \geq 1}$ in defining signal to noise ratio, which both satisfy $\|\Epsilon_n\|_{op} \leq U_n, \forall n \geq 1$ with probability at least $1 - \delta$, per Lemma~\ref{lem:int:cov_noise}.
\begin{equation*}
    U_n^2 :=\begin{cases} 
    \beta\left(3\sqrt{n\ell_{\delta/2\cN}(n)} + 5\ell_{\delta/2\cN}(n)\right) + n\gamma \quad \text{when Assumption~\ref{ass:noise1} holds} \\
    \frac{3}{2}\sqrt{nCd\gamma\ell_\delta(n)} + \frac{7}{3}Cd\ell_\delta(n) + n \gamma \quad \text{when Assumption~\ref{ass:noise2} holds}.
    \end{cases}
\end{equation*}

%
In the above, $\delta \in (0, 1)$, $\ell_\delta(n) := 2\log\log(2 n) + \log\left(\frac{d\pi^2}{12\delta}\right)$, $\beta = 32\sigma^2 e^2$, and $\cN = 17^d$ is an upper bound on the $1/8$-covering number of $\S^{d- 1}$. 
%
%
While the exact forms of the above sequences $(U_n)_{n \geq 1}$ may appear complicated, it is helpful to realize that, under either Assumption~\ref{ass:noise1} or \ref{ass:noise2}, we have $U_n = O\left(\sqrt{n} + \sqrt{d} + \sqrt{\log\left(\frac{1}{\delta}\log(n)\right)}\right)$, i.e., the growth condition on $U_n$ presented in Definition~\ref{def:snr} is satisfied.

We can likewise define the \textit{empirical signal to noise ratio associated with action $a$} as $\wh{\snr}_n(a) := \frac{\sigma_r(\vZ_n(a))}{U_n}$. 
Note that unlike the (true) signal to noise ratio $\snr_n(a)$, the empirical version $\wh{\snr}_n(a)$ is \emph{computable} by the learner. 
Thus, it will be integral in stating our empirical-style bounds in the section that follows. 
%

%
%
\looseness-1

We conclude this section by comparing our notion of signal to noise ratio to that of \citet{agarwal2020principal}, who define $\snr_n(a)$ instead as $\frac{\sigma_r(\vX_n(a))}{\sqrt{n} + \sqrt{d}}$, i.e. the ratio of the ``signal'' in the covariates to the \emph{expected} operator norm of covariate noise $\E\|\Epsilon_n\|_{op}$. 
Since the goal of our work is high-probability (not in-expectation) estimation guarantees for PCR, we believe using high probability bounds on $\|\Epsilon_n\|_{op}$ is more natural when defining $\snr_n(a)$. 
\section{Adaptive Bounds for Principal Component Regression}\label{sec:oPCR}

We now present the main results of this work---high-probability, time- and action-uniform bounds measuring the convergence of the PCR estimates $\wh{\theta}_n(a)$ to the true slope vectors $\theta(a)$. Unlike existing results~\citep{agarwal2019robustness, agarwal2020principal,agarwal2020synthetic}, our bounds are valid when the covariates $(X_n)_{n \geq 1}$ and actions $(a_n)_{n \geq 1}$ are determined in an online (potentially adversarial) manner. 

We first point out why the analysis of \citet{agarwal2020principal, agarwal2020synthetic} breaks down in the setting of adaptive (or online) PCR. 
First, many of the concentration inequalities leveraged in \citet{agarwal2020principal} do not hold in the adaptive design setting. 
As a particular example, the authors leverage the Hanson-Wright inequality \citep{vershynin2018high, rudelson2013hanson} for quadratic forms to study how the noisy covariate matrix $\vZ_n$ concentrates around the true matrix $\vX_n$. This inequality fails to hold when the design points $(X_n)_{n \geq 1}$ depend on the previous observations. Second, the techniques leveraged by \citet{agarwal2020synthetic} to extend the convergence guarantees of PCR to the multiple action setting fail to hold when the $n$-th action $a_n$ is selected based on previous observations.
Lastly, the bounds presented in \cite{agarwal2020principal} are are inherently fixed-time in nature---a simple way to convert existing fixed-time bounds to time-uniform ones would be to perform a union bound over time steps, but that introduces looseness in the bounds.

We are able to construct our bounds by exploiting connections between online PCR and self-normalized martingale concentration~\citep{howard2020time, howard2021time, de2004self, de2007pseudo}. In particular, we combine martingale-based results for constructing confidence ellipsoids for online regression \citep{abbasi2011improved, de2004self, de2007pseudo} with methods for high-dimensional covariance estimation~\citep{wainwright2019high, tropp2015introduction} to prove our results. 
Exploiting this connection is what allows us to extend the results of \citet{agarwal2020principal} to the adaptive design, time-uniform setting. 
We begin with a bound which, up to constants and polylogarthmic factors, captures the rate of convergence of online PCR in terms of (a) the underlying signal to noise ratio and (b) the conditioning of the observed data. 

\begin{theorem}[\textbf{Rate of Convergence for Online PCR}]
\label{thm:nice_bd}
Let $\delta \in (0, 1)$ be an arbitrary confidence parameter. Let $\rho > 0$ be chosen to be sufficiently small, as detailed in Appendix~\ref{app:ridge}. Further, assume that there is some $n_0 \geq 1$ such that $\rank(\vX_{n_0}(a)) = r$ and $\snr_n(a) \geq 2$ for all $n \geq n_0$. Then, with probability at least $1 - O(A\delta)$, simultaneously for all actions $a \in [A]$ and time steps $n \geq n_0$, we have
\[
\|\wh{\theta}_n(a) - \theta(a)\|_2^2 = \wt{O}\left(\frac{1}{\snr_n(a)^2}\kappa(\vX_n(a))^2\right),
\]
where $\kappa(\vX_n(a)) := \frac{\sigma_1(\vX_n(a))}{\sigma_r(\vX_n(a))}$ is the condition number (ignoring zero singular values) of $\vX_n(a)$.
\end{theorem}
Theorem~\ref{thm:nice_bd} is proved in Appendix~\ref{app:nice}. 
We begin by comparing our bounds to those of \citet{agarwal2020principal, agarwal2020synthetic}. At any fixed time, our bounds take on roughly the same form as those of the aforementioned authors, having an inverse quadratic dependence on the signal to noise ratio. To make their bounds non-vacuous, the authors need to make the ``soft sparsity'' assumption of $\|\theta(a)\|_1 = O(\sqrt{d})$. Our bound, on the other hand, suffers no dependence on the $\ell_1$-norm of the $\theta(a)$'s. 
This makes intuitive sense, as the specific choice of a basis should not impact the rate of convergence of PCR. 
However, our bounds pay a price for adaptivity---in particular, the signal to noise ratio associated with an action is defined with respect to a bound on the \textit{total} operator norm of the matrix $\Epsilon_n$. 
If an action is selected very infrequently, the above bound may become loose.\looseness-1

While the above bound is stated in terms of signal to noise ratio, if we make additional assumptions, we can obtain bounds directly in terms of $d, n,$ and $r$. 
In particular, the following ``well-balancing'' assumptions suffice.
%
%
\begin{assumption}[\textbf{Well-balancing assumptions}]
\label{ass:spectrum}
For all $n \geq n_0$, the following hold: (a) $\sigma_i(\vX_n(a)) = \Theta\left(\sqrt{\frac{c_n(a)d}{r}}\right)$ for all $i \in [r]$, (b) $c_n(a) = \Theta(c_n(a'))$ for all $a, a' \in [A]$, and (c) $A = O(r)$.

\end{assumption}

\begin{corollary}
\label{cor:simp_bd}
Assume the same setup as Theorem~\ref{thm:nice_bd}, and further assume Assumption~\ref{ass:spectrum} holds. Then with probability at least $1 - O(A\delta)$, simultaneously for all actions $a \in [A]$ and time steps $n \geq n_0$, we have
\[
\|\wh{\theta}_n(a) - \theta(a)\|_2^2 = \wt{O}\left(\frac{r^2}{d \land n}\right).
\]

\end{corollary}

Corollary~\ref{cor:simp_bd} shows that Theorem~\ref{thm:nice_bd} obtains the same estimation rate as Theorem 4.1 of \citet{agarwal2020principal} if assumption Assumption~\ref{ass:spectrum} holds. 
This ``well-balancing'' assumption says roughly that all non-zero singular values of $\vX_n$ are of the same order, each action is selected with the same frequency, and that the number of actions is, at most, proportional to dimension of the true, unknown subspace. 
As noted by \citet{agarwal2020principal}, the assumption of a ``well-balanced spectrum'' (for $\vX_n$) is common in many works in econometrics and robust statistics, and additionally holds with high probability if the entries of $\vX_n$ are i.i.d.\citep{loh2011high, bai2021matrix, fan2018eigenvector}. Further, it is often the case that there only few available actions (for instance, in the synthetic control literature there are only two actions \citep{abadie2010synthetic, abadie2003economic, farias2022synthetically}), justifying the assumption that $A = O(r)$. Lastly, ensuring that each action is played (very roughly) the same number of times can be viewed as a price for adaptivity.

The proof of Theorem~\ref{thm:nice_bd} is immediate as a corollary from the following, more complicated bound. Theorem~\ref{thm:emp_bd} below measures the convergence of $\wh{\theta}_n(a)$ to $\theta(a)$ in terms of empirical (i.e. observed) quantities. We imagine this bound to be the most practically relevant of our results, as, unlike the results of \citet{agarwal2020principal}, it is directly computable by the learner, involves known constants, and places minimal conditions on the signal to noise ratio. 

\begin{theorem}[\textbf{Empirical Guarantees for Online PCR}]
\label{thm:emp_bd}
Let $\delta \in (0, 1)$ be an arbitrary confidence parameter. Let $\rho > 0$ be chosen to be sufficiently small, as detailed in Appendix~\ref{app:ridge}. Further, assume that there is some $n_0 \geq 1$ such that $\rank(\vX_{n_0}(a)) = r$ and $\snr_n(a) \geq 2$ for all $n \geq n_0$. Then, with probability at least $1 - O(A\delta)$, simultaneously for all actions $a \in [A]$ and time steps $n \geq n_0$, we have
\begin{align*}
&\left\|\wh{\theta}_n(a) - \theta(a) \right\|_2^2 \leq  \frac{L^2}{\wh{\snr}_n(a)^2}\left[74 + 216\kappa(\vZ_n(a))^2\right] + \frac{2\err_n(a)}{\sigma_r(\vZ_n(a))^2},
\end{align*}
where $\kappa(\vZ_n(a)) := \frac{\sigma_1(\vZ_n(a))}{\sigma_r(\vZ_n(a))}$, $\|\theta(a)\|_2 \leq L$, and in the above we define the ``error'' term $\err_n(a)$ to be\looseness-1
\begin{equation*}
    \begin{aligned}
        \err_n(a) &:= 32\rho L^2 +64\eta^2 \left(\log\left(\frac{A}{\delta}\right) + r\log\left(1 + \frac{\sigma_1(\vZ_n(a))^2}{\rho}\right)\right)\\
        &+ 6\eta^2\sqrt{2c_n(a)\ell_\delta(c_n(a))} + 10\eta^2\ell_\delta(c_n(a)) + 6c_n(a)\alpha.
    \end{aligned}
\end{equation*}
\end{theorem}

We see that the above bound, with the exception of the third term, more or less resembles the bound presented in Theorem~\ref{thm:nice_bd}, just written in terms of the observed covariates $\vZ_n(a)$ instead of the true covariates $\vX_n(a)$.
We view the third term as a slowly growing ``error'' term. 
In particular, all terms in the quantity $\err_n(a)$ are either constant, logarithmic in the singular values of $\vZ_n(a)$, or linear in $c_n(a)$, the number of times by round $n$ action $a$ has been selected. 
This implies that $\err_n(a) = \wt{O}(n + d)$, ensuring $\err_n(a)$ is dominated by other terms in the asymptotic analysis. 
We now provide the proof of Theorem~\ref{thm:emp_bd}. 
The key application of self-normalized, martingale concentration comes into play in bounding the quantities that appear in the upper bounds of terms $T_1$ and $T_2$ (to be defined below). 
\begin{proof}
Observe the decomposition, for any $n \geq 1$ and $a \in [A]$
\[
\wh{\theta}_n(a) - \theta(a) = \wh{\vP}_n(a)(\wh{\theta}_n(a) - \theta(a)) + (\vP^{\perp} - \wh{\vP}_n^{\perp}(a))\theta(a),
\]
where $\vP^{\perp}$ is the projection onto the subspace orthogonal to $W^*$ and $\wh{\vP}_n^{\perp}(a)$ is the projection onto the subspace orthogonal to the learned subspace (i.e. that spanned by $\vZ_{n,r}(a)$). 
Since $\wh{\vP}_n(a)(\wh{\theta}_n(a) - \theta(a))$ and $(\vP^{\perp} - \wh{\vP}_n^{\perp}(a))\theta(a)$ are orthogonal vectors, we have
\[
\left\|\wh{\theta}_n(a) - \theta(a)\right\|_2^2 = \left\|\wh{\vP}_n(a)(\wh{\theta}_n(a) - \theta(a))\right\|^2_2 + \left\|(\wh{\vP}_n^{\perp}(a) - \vP^\perp)\theta(a)\right\|_2^2.
\]
We bound these two terms separately, beginning with the second term.
Going forward, fix an action $a \in [A]$. Observe that with probability at least $1 - \delta$, simultaneously for all $n \geq n_0(a)$,

\begin{align*}
\left\|(\wh{\vP}_n^{\perp}(a)  - \vP^\perp)\theta(a)\right\|_2^2 & \leq \left\|\wh{\vP}_n^{\perp}(a) - \vP^{\perp}\right\|_{op}^2\left\|\theta(a)\right\|_2^2 \\
&\leq L^2 \left\|\wh{\vP}_n^{\perp}(a) - \vP^{\perp}\right\|_{op}^2 = L^2 \left\|\wh{\vP}_n(a) - \vP\right\|_{op}^2 \\
&\leq  \frac{4L^2U_n^2}{\sigma_r(\vX_n(a))^2} \leq \frac{6L^2U_n^2}{\sigma_r(\vZ_n(a))^2},
\end{align*}
where the equality in the above comes from observing $\|\wh{\vP}_n^\perp(a) - \vP^\perp\|_{op} = \|\wh{\vP}_n(a) - \vP\|_{op}$, the second-to-last last inequality comes from applying Lemma~\ref{lem:proj}, and the last inequality follows from the second part of Lemma~\ref{lem:int:sv_conc}. 

We now bound the first term. Observe that we can write 
\begin{equation}\label{eq:proj}
\begin{aligned}
&\left\|\wh{\vP}_n(a)\left(\wh{\theta}_n(a) - \theta(a)\right)\right\|_2^2 \leq \frac{1}{\sigma_r(\vZ_n(a))^2}\left\|\wh{\vZ}_n(a)\left(\wh{\theta}_n(a) - \theta(a)\right)\right\|_2^2 \\
&\leq \frac{2}{\sigma_r(\vZ_n(a))^2}\left[\underbrace{\left\|\wh{\vZ}_n(a)\wh{\theta}_n(a) - \vX_{n}(a)\theta(a)\right\|_2^2}_{T_1} + \underbrace{\left\|\vX_n(a) \theta(a) - \wh{\vZ}_n(a)\theta(a)\right\|_2^2}_{T_2}\right],
\end{aligned}
\end{equation}
where the first inequality follows from the fact that $\wh{\vP}_n(a) \preceq \frac{1}{\sigma_r(\wh{\vZ}_n(a))^2}\wh{\vZ}_n(a)^\top \wh{\vZ}_n(a)$ and $\sigma_r(\vZ_n(a)) = \sigma_r(\wh{\vZ}_n(a))$, and the second inequality comes from applying the parallelogram inequality. First we bound $T_1$. We have, with probability at least $1 - O(\delta)$, simultaneously for all $n \geq n_0(a)$
\begin{equation}\label{eq:T_1}
\begin{aligned}
T_1 &\leq 8\left\|\wc{\cV}_n(a)^{1/2}\left(\wc{\theta}_n(a) - \theta(a)\right)\right\|_2^2 + 6 \left\|\Xi_n(a)\right\|_2^2 + 8\left\|\wh{\vZ}_n(a)\theta(a) - \vX_n(a)\theta(a)\right\|_2^2 \\
&\leq 32\rho L^2 + 64\eta^2\left(\log\left(\frac{A}{\delta}\right) +  r\log\left(1 + \frac{\sigma_1(\vZ_n(a))^2}{\rho}\right)\right) + 16L^2U_n^2\\
&+ 6\eta^2\sqrt{2c_n(a)\ell_\delta(c_n(a))} + 10\eta^2\ell_\delta(c_n(a)) + 6c_n(a)\alpha + 8T_2,
\end{aligned}
\end{equation}
where the first inequality follows from Lemma~\ref{lem:t1}, and the second inequality follows from applying Lemmas~\ref{lem:ellipsoid} and \ref{lem:resp_bd}. $\ell_\delta(n) = 2\log\log(2 n) + \log\left(\frac{d\pi^2}{12\delta}\right)$, as defined in Lemma~\ref{lem:howard:mixture}. 
We now bound $T_2$. 
With probability at least $1 - O(\delta)$ simultaneously for all $n \geq n_0$, we have\looseness-1
\begin{equation}\label{eq:T_2}
\begin{aligned}
T_2 &\leq  2L^2\sigma_1(\vZ_n(a))^2\left\|\vP - \wh{\vP}_n(a)\right\|_{op}^2 + 2L^2\left\|\Epsilon_n\right\|_{op}^2 \\
&\leq \frac{8L^2\sigma_1(\vZ_n(a))^2U_n^2}{\sigma_r(\vX_n(a))^2} + 2L^2U_n^2 \\
&\leq  \frac{12L^2\sigma_1(\vZ_n(a))^2U_n^2}{\sigma_r(\vZ_n(a))^2} + 2L^2U_n^2.
\end{aligned}
\end{equation}
The first inequality follows from \Cref{lem:t2}, the second inequality follows from applying Lemmas~\ref{lem:proj} and \ref{lem:int:cov_noise}, and the final inequality follows from Lemma~\ref{lem:int:sv_conc}.

Piecing the above inequalities together yields the desired result, which can be checked via the argument at the end of Appendix~\ref{app:lem_main}. 
A union bound over actions then yields that the desired inequality holds over all actions $a \in [A]$ with probability at least $1 - O(A\delta)$.
\end{proof}

\section{Applications to Causal Inference with Panel Data}\label{sec:panel}
We now apply our bounds for adaptive PCR to online experiment design in the context of panel data. 
In this setting, the learner is interested in estimating \emph{unit-specific counterfactuals} under different \emph{interventions}, given a sequence of unit \emph{outcomes} (or \emph{measurements}) over \emph{time}. 
Units can range from medical patients, to subpopulations or geographic locations. 
Examples of interventions include medical treatments, discounts, and socioeconomic policies. 

We consider a panel data setting in which the principal observes units over a sequence of rounds. 
In round $n$, the learner observes a unit $n$ under \emph{control} for $T_0 \in \mathbb{N}$ time steps, followed by one of $A$ \emph{interventions} (including control, which we denote by $0$) for the remaining $T - T_0$ time steps, where $\mathbb{N} \ni T > T_0$. 
Overloading notation to be consistent with the literature on panel data, we denote the potential outcome of unit $n$ at time $t$ under intervention $a$ by $\pv{Y}{a}_{n,t} \in \mathbb{R}$, the set of unit $n$'s pre-treatment outcomes (under control) by $Y_{n,pre} := [Y_{n,1}^{(0)}, \ldots, Y_{n,T_0}^{(0)}]^{\top} \in \mathbb{R}^{T_0}$, and their post-intervention potential outcomes under intervention $a$ by $\pv{Y}{a}_{n,post} := [Y_{n,T_0+1}^{(a)}, \ldots, Y_{n,T}^{(a)}]^{\top} \in \mathbb{R}^{T-T_0}$. 
We use $a$ to refer to an arbitrary intervention in $\{0, \ldots, A-1\}$ and $a_n$ to denote the \emph{realized} intervention unit $n$ actually receives in the post-intervention time period.
We posit that potential outcomes are generated by the following \emph{latent factor model} over units, time steps, and interventions.\looseness-1
\begin{assumption}[\textbf{Latent Factor Model}]\label{ass:lfm}
Suppose the outcome for unit $n$ at time step $t$ under treatment $a \in \{0, \ldots, A-1\}$ takes the form 
\begin{equation*}
    \pv{Y}{a}_{n,t} = \langle U_t^{(a)}, V_n \rangle + \epsilon_{n,t}^{(a)},
\end{equation*}
where $U_t^{(a)} \in \mathbb{R}^r$ is a latent factor which depends only on the time step $t$ and intervention $a$, $V_n \in \mathbb{R}^r$ is a latent factor which only depends on unit $n$, and $\epsilon_{n,t}^{(a)}$ is zero-mean SubGaussian random noise with variance at most $\sigma^2$. 
We assume, without loss of generality, that $|\langle U_t^{(a)}, V_n \rangle| \leq 1$ for all $n \geq 1$, $t \in [T]$, $a \in \{0, \ldots, A-1\}$.
\end{assumption}
Note that the learner observes $\pv{Y}{a}_{n,t}$ \emph{for only the intervention $a_n$ that unit $n$ is under at time step $t$}, and never observes $U_t^{(a)}$, $V_n$, or $\epsilon_{n,t}^{(a)}$. 
Such ``low rank'' assumptions are ubiquitous within the panel data literature (see references in~\Cref{sec:related}). 
We assume that $r$ is known to the learner, although principled heuristics exist for estimating $r$ in practice from data (see, e.g. Section 3.2 of \citet{agarwal2020synthetic}).
Additionally, we make the following ``causal transportability'' assumption on the latent factors. 
\begin{assumption}[\textbf{Linear span inclusion}]\label{ass:hlsi}
    For any post-intervention time step $t \in [T_0+1, T]$ and intervention $a \in \{0, \ldots, A-1\}$, we assume that $\pv{U}{a}_t \in \spn(\{\pv{U}{0}_{t} : t \in [T_0]\})$.
\end{assumption}
Intuitively,~\Cref{ass:hlsi} allows for information to be inferred about the potential outcomes in the post-intervention time period using pre-treatment observations. 
A popular goal in the literature is to estimate unit-specific counterfactual outcomes under different interventions. 
In the sequel, we will show how to do this \emph{when the sequence of units and interventions is chosen adaptively}. 
In line with previous work, our target causal parameter is the (counterfactual) \emph{average expected post-intervention outcome}.  
\begin{definition}\label{def:apio}
    The average expected post-intervention outcome of unit $n$ under intervention $a$ is\looseness-1
    \begin{equation*}
        \E \pv{\Bar{Y}}{a}_{n,post} := \frac{1}{T-T_0} \sum_{t=T_0+1}^T \E \pv{Y}{a}_{n,t},
    \end{equation*}
    where the expectation is taken with respect to $(\epsilon_{n,t}^{(a)})_{T_0 < t\leq T}$.
\end{definition}
While we consider the \emph{average} post-intervention outcome, our results may be readily extended to settings in which the target causal parameter is any \emph{linear} combination of post-intervention outcomes. 

The remainder of this section proceeds as follows: In~\Cref{sec:sc}, we provide finite sample guarantees for the \emph{synthetic interventions} estimator~\citep{agarwal2020synthetic}, a generalization of the popular \emph{synthetic control} method for estimating counterfactuals from panel data~\citep{abadie2003economic, abadie2010synthetic}, when interventions are assigned adaptively. 
In~\Cref{sec:htt}, we provide a method for learning an intervention assignment \emph{policy} (i.e., a mapping from pre-treatment outcomes to interventions) with provably good performance, as measured by \emph{regret}. 
\subsection{Adaptive Synthetic Control}\label{sec:sc}
\emph{Synthetic control (SC)} is a popular framework used to estimate counterfactual unit outcomes in panel data settings, had they not been treated (i.e. remained under \emph{control}) \cite{abadie2003economic, abadie2010synthetic}.
In SC, there is a notion of a \emph{pre-intervention} time period in which all units are under control, followed by a \emph{post-intervention} time period, in which every unit undergoes one of several interventions (including control). 
At a high level, SC fits a model of a unit's pre-treatment outcomes using pre-treatment data from units who remained under control in the post-intervention time period.
It then constructs a ``synthetic control'' by using the learned model to predict the unit's post-intervention outcomes, had they remained under control.
\emph{Synthetic interventions (SI)} is a recent generalization of the SC framework, which allows for counterfactual estimation of unit outcomes under different interventions, in addition to control \cite{agarwal2020synthetic}. 
Using our bounds from~\Cref{sec:oPCR}, we show how to generalize the SI framework of~\citet{agarwal2020synthetic} to settings where interventions are assigned via an \emph{adaptive intervention assignment policy}.

As a motivating example, consider an online e-commerce platform (learner) which assigns discounts (interventions) to users (units) with the goal of maximizing total user engagement on the platform. 
For concreteness, suppose that the e-commerce platform assigns discounts \emph{greedily} with respect to the discount level which appears to be best at the current round (i.e. maximizes total engagement for the current user), given the sequence of previously observed (user, discount level, engagement level) tuples.
Under such a setting, the intervention assigned at the current round $n$ will be correlated with the observed engagement levels at previous rounds, thus breaking the requirement of the SI framework~\cite{agarwal2020synthetic} that the intervention assignment is not adaptive to previously observed outcomes. 
Formally, we provide performance guarantees for the following procedure, which may be thought of as a regularized version of the synthetic interventions estimator of~\citet{agarwal2020synthetic}: 
\begin{definition}[(Regularized) Synthetic Interventions]\label{def:si}
    Given a test unit $n$ set of donor units $\cI(a)$ who have received intervention $a$ in the post-treatment time period (where $c_n(a) = |\cI(a)|$), 
    \begin{enumerate}
        \item Learn a linear relationship between the test unit and the donor units using PCR. 
        Specifically, let $\vZ_n(a) = [Y_{i, pre}^{\top}]_{i \in \cI(a)} \in \mathbb{R}^{T_0 \times c_n(a)}$, $\vY_n(a) = Y_{n,pre} \in \mathbb{R}^{T_0}$, and compute 
        \begin{equation*}
            \wh{\theta}_n(a) := \wh{\cV}_n(a)^{-1}\wh{\vZ}_n(a)\vY_n(a),
        \end{equation*}
        where $\wh{\cV}_n(a)$, $\wh{\vZ}_n(a)$ are defined as in~\Cref{def:pcr}.
        \item Estimate $\E \Bar{Y}_{n, post}^{(a)}$ by $$\wh \E \Bar{Y}_{n, post}^{(a)} = \frac{1}{T - T_0} \sum_{t=T_0 + 1}^T \langle \wh \theta_n(a), Y_{\cI(a),t}^{(a)} \rangle,$$ where $Y_{\cI(a),t}^{(a)} = [Y_{i,t}^{(a)}]^{\top}_{i \in \cI(a)} \in \mathbb{R}^{c_n(a)}$.
    \end{enumerate}
\end{definition}
\begin{theorem}[\textbf{Prediction error; regularized synthetic interventions}]\label{thm:si}
    Let $\delta \in (0, 1)$ be an arbitrary confidence parameter and $\rho > 0$ be chosen to be sufficiently small, as detailed in Appendix~\ref{app:ridge}. 
    Further, assume that Assumptions \ref{ass:lfm} and \ref{ass:hlsi} are satisfied, there is some $n_0 \leq n$ such that $\rank(\vX_{n_0}(a)) = r$. 
    Under~\Cref{ass:spectrum} with probability at least $1 - O(A\delta)$, simultaneously for all interventions $a \in \{0, \ldots, A-1\}$
    \begin{equation*}
        |\widehat{\E} \pv{\Bar{Y}}{a}_{n,post} - \E \pv{\Bar{Y}}{a}_{n,post}| = \Tilde{O} \left( \frac{r^2 \sqrt{n}}{n \wedge T_0} + \frac{r^2}{\sqrt{n \wedge T_0}} + \frac{r}{\sqrt{(T - T_0) (n \wedge T_0)}} \right) 
    \end{equation*}
    where $\widehat{\E}\pv{\Bar{Y}}{a}_{n,post}$ is the estimated average post-intervention outcome for unit $n$ under intervention $a$ given by the regularized synthetic interventions estimator (\Cref{def:si}). 
\end{theorem}
A more complicated expression which does not require~\Cref{ass:spectrum} may be found in~\Cref{app:panel}. 
Observe that $|\widehat{\E} \pv{\Bar{Y}}{a}_{n,post} - \E \pv{\Bar{Y}}{a}_{n,post}| \rightarrow 0$ with high probability as $T_0, n \rightarrow \infty$. 
\begin{proof}[Proof Sketch]
    Our proof proceeds by breaking the prediction error into three terms, in a manner which is similar to~\citet{agarwal2020principal}. 
    Specifically, we have two terms which depend on the noise in the post-intervention outcomes, and one term which depends on the error in learning the relationship between the test and donor units using the pre-treatment data. 
    While the two ``noise'' terms may be bound using straightforward applications of~\Cref{thm:nice_bd} and the Azuma-Hoeffding inequality, bounding the third term requires a non-trivial calculation where we leverage lemmas \ref{lem:proj},\ref{lem:int:sv_conc}, \ref{lem:l2-panel} and a line of reasoning similar to equations (\ref{eq:proj}), (\ref{eq:T_1}), (\ref{eq:T_2}) in the proof of~\Cref{thm:emp_bd}. 
\end{proof}
\section{Learning How to Treat}\label{sec:htt}
In this section, we show how to leverage our bounds for adaptive PCR in order to design new algorithms for decision-making in panel data settings.\footnote{While we focus on panel data, our results in this section are also applicable to the (more general) setting in which a learner is faced with a contextual bandit problem with noise/measurement error in the contexts.} 
Here we study a setting in which $N$ units arrive \emph{sequentially}. 
In each round $n \in [N]$, we (1) observe unit $n$'s pre-treatment outcomes $Y_{n,pre}$, (2) assign an intervention $a_n \in \{0, 1, \ldots, A-1\}$, and (3) observe post-intervention outcomes $Y_{n,post}^{(a_n)}$. 
We assume that the decision-maker's objective in this setting is to maximize the average post-intervention outcome for each unit (i.e., assign the intervention $a_n^* = \arg\max_{a \in [A]} \E \Bar{Y}_{n,post}^{(a)}$ to unit $n$), although our results readily extend to the setting in which the decision-maker wants to optimize any linear function of the post-treatment outcomes.

Our goal in this setting is to learn a good intervention assignment policy (i.e. mapping from pre-treatment outcomes to interventions), as measured by \emph{regret}. 
\begin{definition}[Regret]
    The regret of the decision-maker is the cumulative difference in average expected post-intervention outcomes between the sequence of assigned interventions $a_1, \ldots, a_n$ and the sequence of optimal interventions $a_1^*, \ldots, a_n^*$, where $a_n^* = \arg\max_{a \in [A]} \E \Bar{Y}_{n,post}^{(a)}$. 
    Formally, we say that the decision-maker has regret $R(N, T)$, where 
    \begin{equation*}
        R(N,T) := \sum_{n=1}^N \E \Bar{Y}_{n,post}^{(a_n^*)} - \E \Bar{Y}_{n,post}^{(a_n)}.
    \end{equation*}
\end{definition}
Next we show that under~\Cref{ass:lfm} and~\Cref{ass:hlsi}, $\E \pv{\Bar{Y}}{a}_{n,post}$ may be written as a linear combination of unit $n$'s \emph{pre}-intervention outcomes. 
In addition to being useful for learning a good intervention assignment policy, this reformulation will allow us to make connections to the literature on contextual bandits. 
We note that similar observations have previously been made in the panel data literature (e.g.~\cite{harris2022strategic}), but we include the following lemma for completeness' sake. 
\begin{lemma}[\textbf{Reformulation of average expected post-intervention outcome}]\label{lem:reformulation}
    Under~\Cref{ass:lfm} and~\Cref{ass:hlsi}, there exists slope vector $\theta(a) \in \R^{T_0}$, such that the average expected post-intervention outcome of unit $n$ under intervention $a$ is expressible as 
    \begin{equation*}
        \E \pv{\Bar{Y}}{a}_{n,post} = \frac{1}{T-T_0} \langle \theta(a), \E Y_{n,pre} \rangle.
    \end{equation*}
\end{lemma}
$\theta(a)$ may be interpreted as a unit-independent measure of the causal relationship between pre- and post-intervention outcomes.
For the reader familiar with the literature on contextual bandits, our setting may be thought of as a generalization of the linear contextual bandit setting (see, e.g.~\cite{slivkins2019introduction, lattimore2020bandit}) where there is additional noise in the context. 
Here $\E Y_{n,pre}$ is the (unobserved) context associated with unit $n$, $Y_{n,pre}$ is the (observed) ``noisy'' context associated with unit $n$, $a_n$ is the decision-maker's action, and $\Bar{Y}_{n,post}^{(a_n)}$ is the observed reward. 

Our approach (\Cref{alg:etc}) takes a number $N_0 > 0$ as input and proceeds by using the first $A \cdot N_0$ units for ``exploration'' by assigning each intervention $N_0$ times, before estimating $\theta(a)$ in~\Cref{lem:reformulation} as $\wh{\theta}(a)$ using regularized PCR. 
For $n \in [A \cdot N_0 + 1, \ldots, N]$, \Cref{alg:etc} then assigns intervention $a_n = \arg\max_{a \in [A]} \langle \widehat{\theta}(a), Y_{n,pre} \rangle$ to unit $n$ after observing pre-treatment outcomes $Y_{n,pre}$. 
Observe that this is one example of an intervention assignment policy which results in adaptively-collected data. 

Our results for estimating $\theta(a)$ are of independent interest, so we state them separately from the analysis of~\Cref{alg:etc}. 
We refer to our estimation procedure as ``horizontal regression'' since it regresses over time-steps (in contrast to ``vertical regression'' methods like those in~\Cref{sec:sc}, which regress over units). 
Using a horizontal regression approach to learn how to treat has the benefit that $\wh \theta(a)$ only needs to be computed once (after the first $A \cdot N_0$ units), while an approach based on vertical regression would need to compute a separate $\wh \theta_n(a)$ for each unit $n > A \cdot N_0$. 
\begin{definition}[(Regularized) Horizontal Regression]\label{def:hr}
    Given a test unit $n$ set of donor units $\cI(a)$ who have received intervention $a$ in the post-treatment time period (where $c_n(a) = |\cI(a)|$), 
    \begin{enumerate}
        \item Learn a linear relationship between the pre- and post-intervention outcomes using PCR. 
        Specifically, let
        $\vZ_n(a) = [Y_{i,pre}]_{i \in \cI(a)} \in \mathbb{R}^{c_n(a) \times T_0}$ and 
        $\vY_n(a) = \left[\sum_{t=T_0 + 1}^{T} \pv{Y}{a}_{i,t} \right]_{i \in \cI(a)}^{\top} \in \mathbb{R}^{c_n(a)}$. 
        Estimate $\theta(a) \in \mathbb{R}^{T_0}$ as 
        \begin{equation*}
            \wh{\theta}(a) := \wh{\cV}_n(a)^{-1}\wh{\vZ}_n(a)\vY_n(a),
        \end{equation*}
        where $\wh{\cV}_n(a)$, $\wh{\vZ}_n(a)$ are defined as in~\Cref{def:pcr}. 
        \item Estimate $\E \Bar{Y}_{n, post}^{(a)}$ by $$\widehat{\E}\pv{\Bar{Y}}{a}_{n,post} := \frac{1}{T - T_0} \cdot \langle \wh{\theta}(a), Y_{n,pre} \rangle.$$ 
    \end{enumerate}
\end{definition}
\begin{theorem}[\textbf{Prediction error; horizontal regression}]\label{thm:horz}
    Let $\delta \in (0, 1)$ be an arbitrary confidence parameter and $\rho > 0$ be chosen to be sufficiently small, as detailed in Appendix~\ref{app:ridge}. 
    Further, assume that Assumptions \ref{ass:lfm} and \ref{ass:hlsi} are satisfied, there is some $n_0 \geq 1$ such that $\rank(\vX_{n_0}) = r$, and $\snr_n(a) \geq 2$ for all $n \geq n_0$. 
    Then with probability at least $1 - O(A\delta)$, simultaneously for all interventions $a \in \{0, \ldots, A-1\}$
    \begin{equation*}
        |\widehat{\E} \pv{\Bar{Y}}{a}_{n,post} - \E \pv{\Bar{Y}}{a}_{n,post}| = \Tilde{O} \left( \frac{r^2}{\sqrt{T_0 \wedge n}} + \frac{r^2 \sqrt{T_0}}{\sqrt{T - T_0} (T_0 \wedge n)} \right)
    \end{equation*}
    where $\widehat{\E}\pv{\Bar{Y}}{a}_{n,post}$ is the estimated average post-intervention outcome for unit $n$ under intervention $a$ given by the horizontal regression estimator of~\Cref{def:hr}. 
\end{theorem}
The proof of~\Cref{thm:horz} proceeds similarly to the proof of~\Cref{thm:si}. 
While we still decompose the prediction error into three terms, when compared to~\Cref{thm:si} these terms are ``transposed'', roughly meaning that the roles of units and time-steps are reversed. 
This is to be expected, as for every intervention $a \in \{0, \ldots, A-1\}$, the synthetic interventions method of~\Cref{def:si} regresses over the $c_n(a)$ units and has $T_0$ data points, whereas the horizontal regression method of~\Cref{def:hr} regresses over the $T_0$ time-steps and has $c_n(a)$ data points. 
Regret guarantees for~\Cref{alg:etc} follow readily from~\Cref{thm:horz}. 
\begin{algorithm}[t]
        \SetAlgoNoLine
        \SetAlgoNoEnd
        \KwIn{Explore length $N_0$}
        \For{$a \in [A]$}
        {
            Assign intervention $a$ for $N_0$ rounds\\
            Estimate $\theta(a)$ as $\widehat{\theta}(a)$ using the method of~\Cref{def:hr}
        }
        \For{$n = A \cdot N_0 + 1, \ldots, N$}
        {   
            Assign intervention $a_n = \arg\max_{a \in [A]} \langle \widehat{\theta}(a), Y_{n,pre} \rangle$ to unit $n$
        }
        \caption{Explore Then Intervene}
        \label{alg:etc}
\end{algorithm}
\begin{corollary}
    Suppose that $A = O(r)$, $\rank(\vX_{A \cdot N_0}(a)) = r$, $\snr_{A \cdot N_0}(a) \geq 2$, and $\sigma_i(\vX_{A \cdot N_0}(a)) = \Theta\left(\sqrt{\frac{N_0 T_0}{r}}\right)$ for all $i \in [r]$. 
    Then, with probability at least $1 - O(A\delta)$, the regret of~\Cref{alg:etc} is bounded as 
    \begin{equation*}
        \frac{R(N,T)}{N} = \widetilde{O} \left( \frac{N_0}{N} + \frac{r^3}{\sqrt{T_0 \wedge N_0}} + \frac{r^3\sqrt{T_0}}{\sqrt{T-T_0}(T_0 \wedge N_0)} \right).
    \end{equation*}
\end{corollary}
\begin{proof}
    \begin{equation*}
    \begin{aligned}
        \frac{R(N,T)}{N} 
        &\leq \frac{1}{N} \left( N_0 + \sum_{n=N_0 + 1}^N \E Y_{n,post}^{(a_n^*)} - \widehat{\E} Y_{n,post}^{(a_n^*)} + \widehat{\E} Y_{n,post}^{(a_n)} - \E Y_{n,post}^{(a_n)} + \widehat{\E} Y_{n,post}^{(a_n^*)} - \widehat{\E} Y_{n,post}^{(a_n)} \right)\\
        %
        %
        &\leq \frac{1}{N} \left( N_0 + \sum_{n=N_0 + 1}^N \sum_{a=1}^A |\E Y_{n,post}^{(a)} - \widehat{\E} Y_{n,post}^{(a)}| \right).
    \end{aligned}
    \end{equation*}
    Note that since each intervention is assigned $N_0$ times, all conditions for~\Cref{ass:spectrum} are satisfied. 
    Applying the results of~\Cref{thm:horz} and substituting in for $N_0$ yields the desired result.\looseness-1 
    %
    %
\end{proof}
For intuition, suppose that $T_0 = \Theta(T)$ and we pick $N_0 = \Theta \left( r^2 N^{2/3} \right)$. 
In this case, the regret bound simplifies to 
\begin{equation*}
    \frac{R(N,T)}{N} = \Tilde{O}\left( \frac{r^2}{N^{1/3}} + \frac{r^3}{\sqrt{T \wedge r^2 N^{2/3}}}\right).
\end{equation*}
Observe that with high probability, $\frac{R(N,T)}{N} \rightarrow 0$ as $N,T \rightarrow \infty$. 
\subsection{Simulations}
\begin{figure}[t]
    \centering
    \begin{subfigure}[b]{0.48\textwidth}
        \centering
        \includegraphics[width=\textwidth]{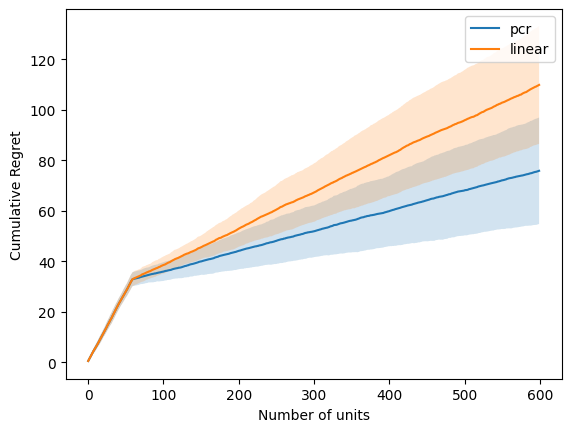}
        \caption{$N_0 = 20$}
    \end{subfigure}
    \hfill
    \begin{subfigure}[b]{0.48\textwidth}
         \centering
        \includegraphics[width=\textwidth]{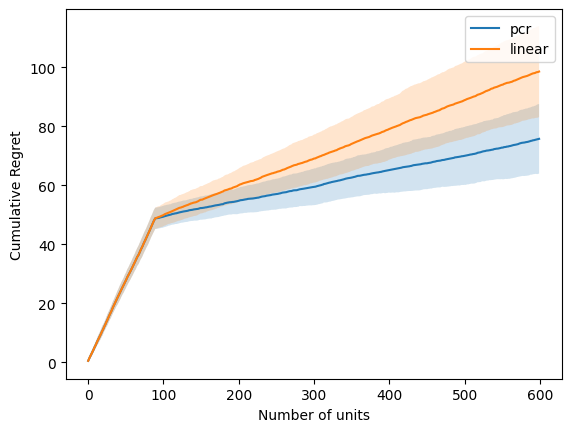}
        \caption{$N_0 = 30$}
    \end{subfigure}
    \hfill
    \begin{subfigure}[b]{0.48\textwidth}
         \centering
        \includegraphics[width=\textwidth]{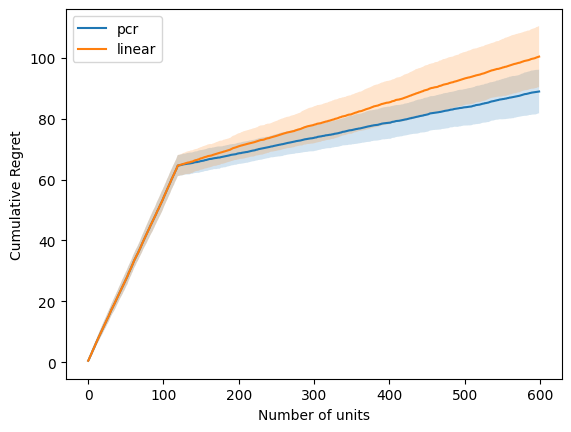}
        \caption{$N_0 = 40$}
    \end{subfigure}
    \hfill
    \begin{subfigure}[b]{0.48\textwidth}
         \centering
        \includegraphics[width=\textwidth]{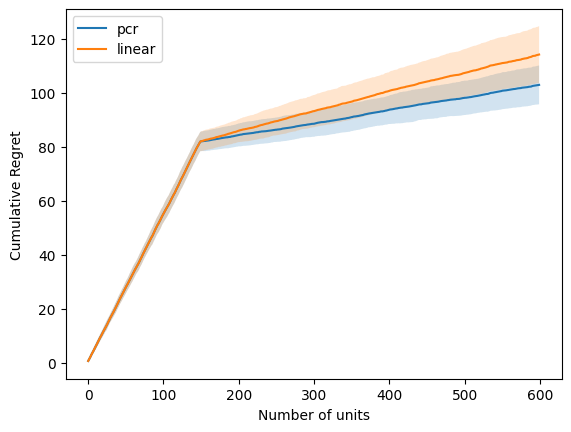}
        \caption{$N_0 = 50$}
    \end{subfigure}
    \caption{Average regret over 50 runs for different values of $N_0$ for~\Cref{alg:etc} (blue) and an ablation which uses linear regression instead of PCR (orange). Shaded regions represent one standard deviation. As $N_0$ decreases, the performance of the linear ablation drops relative to~\Cref{alg:etc}.}
    \label{fig:N0}
\end{figure}
We empirically evaluate the performance of~\Cref{alg:etc} on simulated data.\footnote{It is challenging to evaluate the performance of contextual bandit algorithms on real data due to the fact that the decision-maker observes only bandit feedback. In our case, the decision-maker observes $Y_{n,post}^{(a_n)}$, but not $Y_{n,post}^{(a)}$ for $a \neq a_n$. While one could impute a real-world dataset, this would still introduce a synthetic component.} 
We consider a setting with three interventions, $r = 3$, $T = 20$, $T_0 = 10$, and $N = 600$. 
We generate the latent factors in a way which ensures that each intervention is optimal roughly $1/3$ of the time. 
Specifically, we consider three unit ``types'' $V(0) = [1, 0.1, 0.1]$, $V(1) = [0.1, 1, 0.1]$, and $V(2) = [0.1, 0.1, 1]$, and three time vectors $U(0) = [1, 0.1, 0.1]$, $U(1) = [0.1, 1, 0.1]$, and $U(2) = [0.1, 0.1, 1]$. 
For $t \in [T_0]$, the corresponding latent factor is chosen uniformly-at random-from $\{U(0), U(1), U(2)\}$. 
For $t \in \{T_0 + 1, \ldots, T\}$, the latent factor associated with intervention $a$ is $U_t^{(a)} = U(a)$. 
For every $n \in [N]$, $V_n$ is chosen uniformly-at-random from $\{V(0), V(1), V(2)\}$. 
Finally, the (potential) outcome for unit $n$ at time-step $t$ under intervention $a$ is generated by adding Gaussian noise with standard deviation $\sigma > 0$ to $\langle U_{t}^{(a)}, V_n \rangle$. 
The latent factors are designed in such a way that intervention $a$ is the optimal intervention to assign to a unit with latent factor $V(a)$. 

\paragraph{Experiment $1$: Changing $N_0$.} 
In this experiment, we fix $\sigma = 0.5$ and study the effects of varying the number of explore units $N_0$ on~\Cref{alg:etc} and an ablation which uses linear regression instead of PCR. 
Our results are summarized in~\Cref{fig:N0}. 
As $N_0$ decreases, we see that the performance of the ablation becomes worse relative to the performance of~\Cref{alg:etc}. 
This is because for smaller $N_0$, the noise in the unit outcomes does not concentrate and so methods which do not explicitly de-noise perform poorly. 

\paragraph{Experiment $2$: Changing $\sigma$.} 
\begin{figure}[t]
    \centering
    \begin{subfigure}[b]{0.48\textwidth}
        \centering
        \includegraphics[width=\textwidth]{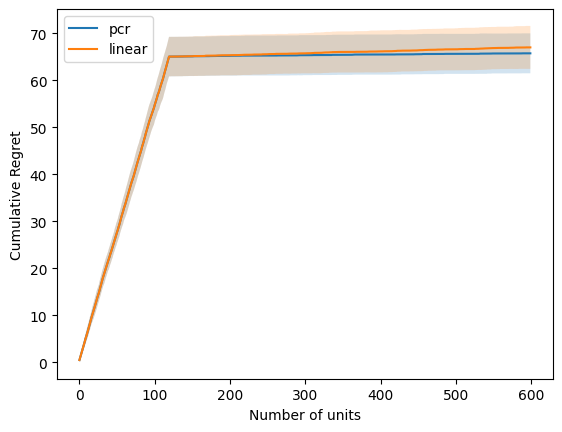}
        \caption{$\sigma = 0.3$}
    \end{subfigure}
    \hfill
    \begin{subfigure}[b]{0.48\textwidth}
         \centering
        \includegraphics[width=\textwidth]{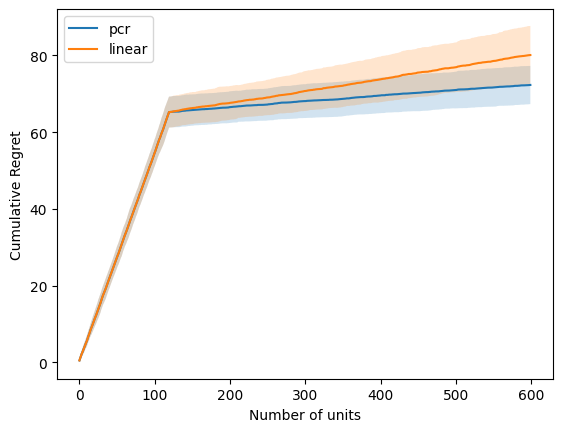}
        \caption{$\sigma = 0.4$}
    \end{subfigure}
    \hfill
    \begin{subfigure}[b]{0.48\textwidth}
         \centering
        \includegraphics[width=\textwidth]{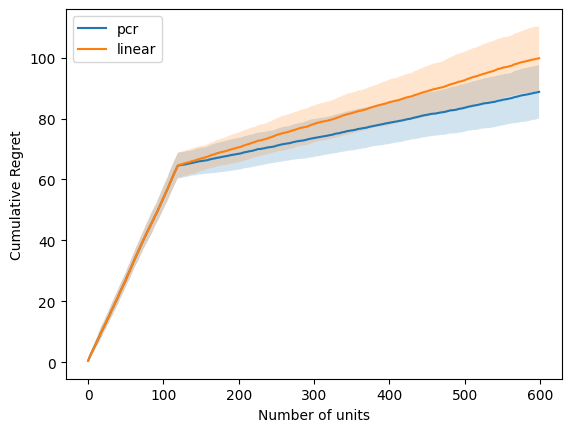}
        \caption{$\sigma = 0.5$}
    \end{subfigure}
    \hfill
    \begin{subfigure}[b]{0.48\textwidth}
         \centering
        \includegraphics[width=\textwidth]{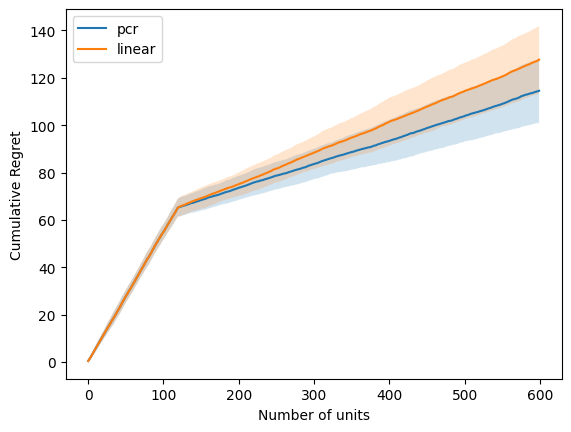}
        \caption{$\sigma = 0.6$}
    \end{subfigure}
    \caption{Average regret over 50 runs for different values of $\sigma$ for~\Cref{alg:etc} (blue) and an ablation which uses linear regression instead of PCR (orange). Shaded regions represent one standard deviation. As $\sigma$ increases, the performance of the linear ablation drops relative to~\Cref{alg:etc}.}
    \label{fig:noise}
\end{figure}
Next we fix $N_0 = 40$ and study the effects of varying the noise standard deviation $\sigma$ on both algorithms. 
We find that as $\sigma$ increases, the performance of both algorithms decreases, although the decrease in the ablation is more than that of~\Cref{alg:etc}. 
Intuitively this is because the need for an explicit de-noising step grows as the magnitude of the noise increases. 
\section{Conclusion}
We obtain the first adaptive bounds for principal component regression and apply them to two problems in the panel data setting. 
The first is that of online experiment design using panel data, where we allow for interventions to be assigned according to an adaptive policy. 
The second is the problem of learning how to assign interventions to units in order to maximize some function of the post-treatment outcomes. 
Here we design an algorithm which achieves no-regret with respect to the optimal intervention assignment policy and empirically evaluate the performance of our algorithm on synthetic data. 
Exciting directions for future work include applications of our results to domains such as differential privacy, and using our bounds to obtain algorithms with better regret guarantees (e.g. based on LinUCB~\cite{abbasi2011improved}) for the linear contextual bandit problem with noisy contexts.
\section*{Acknowledgements}
KH is supported in part by an NDSEG Fellowship. 
KH, JW, and ZSW are supported in part by NSF FAI Award \#1939606. 
The authors would like to thank anonymous reviewers for valuable feedback.

\newpage 
\bibliographystyle{plainnat}
\bibliography{refs}

\newpage
\appendix
\section{Results on Martingale Concentration}
\label{app:martingale}

In this section, we discuss the basics of martingale concentration that will be used ubiquitously throughout this work. Recall that a process $(M_n)_{n \geq 0}$ is a martingale with respect to some filtration $(\calF_n)_{n \geq 0}$ if (a) $(M_n)_{n \geq 0}$ is adapted to $(\calF_n)_{n \geq 0}$, (b) $\E|M_n| < \infty$ for all $n \geq 0$, and (c) $\E(M_{n + 1} \mid \calF_n) = M_n$ for all $n \geq 0$. We denote the ``increments'' of a martingale as $\Delta M_n := M_n - M_{n - 1}$. Results on martingale concentration yield a means of providing tight, time-uniform concentration for statistical tasks in which data is adaptively collected. Recent advances in martingale concentration have allowed for advances in disparate fields of statistical theory such as PAC-Bayesian learning~\citep{chugg2023unified}, composition in differential privacy~\citep{whitehouse2022fully, whitehouse2022brownian}, and the estimation of convex statistical divergences~\citep{manole2023martingale}.

Martingale concentration emerges naturally in our work in two ways. First, martingale concentration plays an integral role in bounding the deviations between $\wc{\theta}_t(a)$, the ridge estimate in the true low-dimensional subspace, and $\theta(a)$. In bounding these deviations, we leverage results on the concentration of self-normalized martingale processes, particularly those of \citet{de2004self, de2007pseudo}. These results have long proven useful for calibrating confidence in online linear regression tasks~\citep{abbasi2011improved, chowdhury2017kernelized}, but to the best of our knowledge, we are the first to couple these results with the low-dimensional estimation innate to PCR.

Second, we leverage martingale methods to control the rate at which PCR's estimate of the projection operator onto the unknown subspace converges. To accomplish this, we couple recent breakthroughs on time-uniform, self-normalized concentration for scalar-valued processes~\citep{howard2020time, howard2021time} with the covering and matrix-CGF approaches for bounding the error in estimates of covariance matrices~\citep{wainwright2019high, roman2005advanced, tropp2015introduction}. While this aspect of our analysis is, more or less, a straightforward merging of two techniques for concentration of measure, it nonetheless requires care due to the technical nature of the machinery being used. 

We start by recounting the time-uniform concentration inequality we leverage for controlling the error in the ridge estimate in the true, low-dimensional subspace. The following result is from \citet{abbasi2011improved}, but is a special case of more general, self-normalized concentration results from \citet{de2004self, de2007pseudo}.

\begin{lemma}[\textbf{Method of Mixtures}]
\label{lem:ext:mixture}
Let $(\calF_t)_{t \geq 0}$ be a filtration. Let $S_t = \sum_{s = 1}^t \epsilon_s X_s$ where $(\epsilon_t)_{t \geq 1}$ is an $(\calF_t)_{t \geq 0}$-adapted $\R$-valued process of $\sigma$-subGaussian random variables and $(X_t)_{t \geq 1}$ is an $(\calF_t)_{t \geq 1}$-predictable $\R^d$-valued process. Let $\rho > 0$ be arbitrary, and let 
$$
\cV_t := \sum_{s = 1}^t X_s X_s^T + \rho I_d.
$$
Let $\delta \in (0, 1)$ be any confidence parameter. Then we have with probability at least $1 - \delta$, simultaneously for all $t \geq 1$,
$$
\left\|\cV_t^{-1/2}S_t\right\|_2 \leq \sigma\sqrt{2\log\left(\frac{1}{\delta}\sqrt{\det(\rho^{-1}\cV_t)}\right)}.
$$
\end{lemma}

We make several brief comments about the above lemma. First, note that while the word ``martingale'' doesn't directly  appear, the process $(S_n)_{n \geq 0}$ is, in fact, a martingale with respect to the filtration $(\calF_n)_{n \geq 0}$. The concentration inequality follows from ``mixing'' a family of martingales based on $(S_n)_{n \geq 0}$ with respect to some suitable probability measure. Second, how we leverage Lemma~\ref{lem:ext:mixture} in conjunction with the low-dimensional structure of the problem comes through the presence of $\det(\rho^{-1}\cV_n)$ in the bound. In particular, if the sequence $(X_n)_{n \geq 0}$ lies in some low-dimensional subspace $W$, with $\dim(W) = r$, then $\rho^{-1}\cV_n$ will have at most $r$ non-unit eigenvalues, and hence $\log\det(\rho^{-1}\cV_n) \approx r \log(\|\rho^{-1}\cV_n\|_{op})$. We exploit this idea further in the sequel.

Now, we discuss the scalar-valued martingale concentration results from \citet{howard2020time, howard2021time} that will be used in our work. Before this, recall that a random variable $X$ is said to be $\sigma$-subGaussian if $\log\E e^{\lambda X} \leq \frac{\lambda^2 \sigma^2}{2}$ for all $\lambda \in \R$. Likewise, we say $X$ is $(\sigma, c)$-subExponential if $\log\E e^{\lambda X} \leq \frac{\lambda^2 \sigma^2}{2}$ for all $|\lambda| < \frac{1}{c}$, and we say $X$ is $(\sigma, c)$-subGamma if $\log\E e^{\lambda X} \leq \sigma^2 \psi_{G, c}(\lambda)$ for all $|\lambda| < \frac{1}{c}$, where $\psi_{G, c}(\lambda) := \frac{\lambda^2}{2(1 - c\lambda)}$. A particularly useful fact in the sequel is that if $X$ is $(\sigma, c)$-subExponential, it is also $(\sigma, c)$-subGamma \citep{howard2020time}.

\begin{lemma}
\label{lem:howard:mixture}
Suppose $(X_n)_{n \geq 1}$ is a sequence of independent, $(\sigma, c)$-subGamma random variables. Let $(S_n)_{n \geq 0}$ be given by $S_n := \sum_{m = 1}^n X_n$, and let $\delta \in (0, 1)$ be arbitrary. Then, with probability at least $1 - \delta$, simultaneously for all $n \geq 0$, we have 
\[
S_n \leq \frac{3}{2}\sigma\sqrt{n\ell_\delta(n)} + \frac{5}{2}c\ell_\delta(n),
\]
where $\ell_\delta(n) := 2\log\log(2 n) + \log\left(\frac{d\pi^2}{12\delta}\right)$.
\end{lemma}

Note that we have simplified above bound from \citet{howard2021time} in several ways. First, the bound as presented in \citet{howard2021time} applies to general classes of processes with potentially correlated increments. We have simplified the bound to the setting of independent, subGamma increments to suit our setting. Further, the original bound has many parameters, each of which can be fine-tuned to fit an application. We have pre-selected parameters so that (a) the bound is legible and (b) constants remain relatively small.

We can couple the above bound with a standard argument for bounding error in covariance estimation \citep{wainwright2019high} to obtain high-probability, time-uniform results. The following result will prove useful in measuring the rapidity at which PCA can learn the true, low-dimensional subspace in which the noiseless contexts and slope vectors lie.

\begin{lemma}
\label{lem:noise:subgaussian}
Let $(\epsilon_n)_{n \geq 1}$ be a sequence of independent, $\sigma$-subGaussian random vectors in $\R^d$. Then, for any $\delta \in (0, 1)$, with probability at least $1 - \delta$, simultaneously for all $n \geq 1$, we have
\[
\left\|\sum_{m = 1}^n \epsilon_m\epsilon_m^\top - \E\epsilon_m\epsilon_m^\top\right\|_{op} \leq \beta\left(3\sqrt{n\ell_{\delta/2\cN}(n)} + 5\ell_{\delta/2\cN}(n)\right),
\]
where $\ell_\delta$ is as defined in Lemma~\ref{lem:howard:mixture}, $\beta = 32\sigma^2 e^2$, and $\cN = 17^d$ is an upper bound on the $1/8$-covering number of $\S^{d- 1}$.

\end{lemma}

\begin{proof}
Define the process $(M_n)_{n \geq 0}$ by $M_n := \sum_{m = 1}^n \epsilon_m\epsilon_m^\top - \E\epsilon_m\epsilon_m^\top$. Using a standard covering argument (more or less verbatim from \citet{wainwright2019high}), we have that, if $K \subset \S^{d - 1}$ is a minimal $1/8$-covering of $\S^{d -1}$, then
\begin{equation}\label{eq:wr-op}
\|M_n\|_{op} \leq 2\max_{\nu \in K}\langle \nu, M_n \nu \rangle 
\end{equation}
It is clear that $(M_n)_{n \geq 0}$ is a Hermitian matrix-valued martingale with respect to the natural filtration $(\calF_n)_{n \geq 0}$ given by $\calF_n := \sigma(\epsilon_m : m \leq n)$. It is thus straightforward to see that, for any $\nu \in \S^{d - 1}$, the process $(M_n^{\nu})_{n \geq 0}$ given by $M_n^\nu := \nu^T M_n \nu$ is a real-valued martingale with respect to this same filtration.

Moreover, a standard argument (see the proof of Theorem 6.5 in \citet{wainwright2019high}) yields that 
\begin{equation}
\label{ineq:cov_bound}
\log\left[\E e^{\lambda\langle \nu, \Delta M_n \nu \rangle }\right] \leq \frac{\lambda^2 \beta^2}{2} \qquad \text{for all } |\lambda| < \frac{1}{\beta},
\end{equation}
where $\beta:= 32 e^2 \sigma^2$. In other words, for any $\nu \in \S^{d - 1}$, the random variable $\langle \nu, \Delta M_n \nu \rangle$ is $\left(\beta, \beta\right)$-sub-Exponential, as outlined above. In particular, per the results of \citet{howard2020time, howard2021time}, this implies that the random variable $\langle \nu, \Delta M_n \nu \rangle$ is $(\beta, \beta)$-subGamma. Applying Lemma~\ref{lem:howard:mixture}, we have, with probability at least $1 - \delta$, simultaneously for all $n \geq 1$ and $\nu \in K$, that
\[
|\langle \nu, M_n \nu \rangle| \leq \frac{3}{2}\beta\sqrt{n\ell_{\delta/2\cN}(n)} + \beta\frac{5}{2}\ell_{\delta/2\cN}(n).
\]
Plugging this into Inequality~\ref{eq:wr-op}, we have that, with probability at least $1 - \delta$, simultaneously for all $n \geq 1$,
\[
\|M_n\|_{op} \leq \beta\left(3\sqrt{n\ell_{\delta/2\cN}(n)} + 5\ell_{\delta/2\cN}(n)\right).
\]
%

\end{proof}

In the case that the sequence of independent noise variables $(\epsilon_n)_{n \geq 1}$ satisfies $\E\epsilon_n\epsilon_n^\top = \Sigma$ and $\|\epsilon_n\|_2 \leq \sqrt{B}$ uniformly in $n$, we can obtain signifcantly tighter bounds (in terms of constants). Of particular import is the following bound from \citet{howard2021time}, which combines their ``stitching approach'' to time-uniform concentration with matrix chernoff techniques \citep{howard2020time, tropp2015introduction, wainwright2019high} to obtain time-uniform bounds on covariance estimation. 
\begin{lemma}
\label{lem:noise:bounded}
Let $(\epsilon_n)_{n \geq 1}$ be a sequence of mean zero, independent random vectors in $\R^d$ such that, for all $n \geq 1$, we have $\E \epsilon_n \epsilon_n^\top = \Sigma$ and $\|\epsilon_n\|_2 \leq \sqrt{Cd}$ almost surely. Further assume $\|\Sigma\|_{op} \leq \gamma$. Then, for any $\delta \in (0, 1)$, we have, with probability at least $1- \delta$, simultaneously for all $n \geq 1$
\[
\left\|\sum_{m = 1}^n\left\{ \epsilon_m \epsilon_m^\top - \Sigma\right\}\right\|_{op} \leq \frac{3}{2}\sqrt{nCd\gamma\ell_\delta(n)} + \frac{7}{3}Cd\ell_\delta(t)
\]
where $\ell_\delta(n) := 2\log\log(2 n) + \log\left(\frac{d\pi^2}{12\delta}\right)$.
\end{lemma}
\section{Equivalent Formulations of Ridge Regression}
\label{app:ridge}

We begin by discussing properties and equivalent formulations of ridge regression, as the estimate produced by (regularized) PCR, $\wh{\theta}_n(a)$, is precisely the ridge estimate of the unknown parameter $\theta(a)$ when restricted the subspace associated with the projection matrix $\wh{\vP}_n$.

\begin{fact}[\textbf{Ridge regression formulation}]
Let $\wh{W}_n$ be the subspace associated with the projection matrix $\wh{\vP}_n$. Then, $\wh{\theta}_n(a)$ satisfies
\[
\wh{\theta}_n(a) = \arg\min_{\theta \in \wh{W}_n}\left\{ \left\| \vZ_n(a) \theta - \vY_n(a)\right\|_2^2 + \frac{\rho}{2}\|\theta\|_2^2\right\}.
\]
That is, $\wh{\theta}_n(a)$ is the solution to $\rho$-regularized ridge regression when estimates are restricted to $\wh{W}_n$.\looseness-1

\end{fact}
Ridge regression may also be represented in the following, constrained optimization format. 
\begin{fact}[\textbf{Constrained formulation of ridge regression}]
\label{fact:ridge_const}
Let $\wh{W}_n$ be the subspace associated with the projection matrix $\wh{\vP}_n$. Then, $\wh{\theta}_n(a)$ satisfies
\[
\wh{\theta}_n(a) = \arg\min_{\theta \in \wh{W}_n : \|\theta\|_2 \leq R_\rho}\left\| \vZ_n(a) \theta - \vY_n(a)\right\|_2^2,
\]
where $R_\rho$ is some constant only depending on $\rho$.
\end{fact}

The larger $\rho$ is, the smaller $R_\rho$ must become. Since we know $\|\theta(a)\|_2 \leq L$, for all $a \in [A]$, throughout the main body and appendix of this paper, we assume that $\rho$ is chosen to be sufficiently small such that our estimates $\wh{\theta}_n(a)$ satisfy $\|\wh{\theta}_n(a)\|_2 \leq L$, i.e. we select $\rho >0$ satisfying $R_\rho \leq L$.
\section{Appendix for~\Cref{sec:oPCR}: Adaptive Bounds for Principal Component Regression}
\subsection{Convergence Results for PCA and Singular Values}
\label{app:pca}

In this appendix, we discuss time-uniform convergence results for adaptive principal component analysis --- one half of the principal component regression (PCR) algorithm. In addition, we analyze how singular values of $\vZ_n$, the noisy covariate matrix, cluster around those of $\vX_n$, the true covariate matrix. We, in the style of \citet{agarwal2020principal}, reduce our study of these quantities to the study of the operator norm of $\Epsilon_n$, which we control using the martingale concentration results outlined in Appendix~\ref{app:martingale}.

Before continuing, we enumerate several linear algebraic facts that are useful in bounding deviations in singular values. 

\begin{lemma}[\textbf{Weyl's Inequality}]
\label{lem:ext:weyl}
Let $\vA, \vB \in \R^{t \times n}$. Then, for any $i \in [t \land n]$, we have
$$
\left|\sigma_i(\vA) - \sigma_i(\vB)\right| \leq \left\|\vA - \vB\right\|_{op}.
$$

\end{lemma}

\begin{lemma}[\textbf{Wedin's Lemma}]
\label{lem:ext:wedin}
Let $\vA, \vB \in \R^{t \times n}$, and suppose $\vA$ and $\vB$ have spectral decompositions 
$$
\vA = \vU \Sigma \vV^T \qquad \text{and} \qquad \vB = \wh{\vU} \wh{\Sigma} \wh{\vV}^T.
$$
Then, for any $r \leq n \land d$, we have 
$$
\max\left\{\left\|\vU_r\vU_r^T - \wh{\vU}_r\wh{\vU}_r^T\right\|_{op}, \left\|\vV_r\vV_r^T - \wh{\vV}_r\wh{\vV}_r^T\right\|_{op}\right\} \leq \frac{2\left\|\vA - \vB\right\|_{op}}{\sigma_r - \sigma_{r + 1}},
$$
where $\sigma_r$ (resp. $\sigma_{r+1}$) is the $r$-th (resp. $r+1$-st) largest singular value of $\vA$.
\end{lemma}

We now prove a time-uniform, high probability bound on the operator norm of $\Epsilon_n$. 
(Recall $\Epsilon_n = (\epsilon_1, \ldots, \epsilon_n)^T$.)
In particular, we prove a bound for two settings --- a looser bound (in terms of constants) which holds when the noise variables $\epsilon_n$ are assumed to be subGaussian, and a tighter, more practically relevant bound that holds when $\epsilon_n$ are assumed to be bounded.

\begin{lemma}[\textbf{Covariance Noise Bound}]
\label{lem:int:cov_noise}
Let $(\epsilon_n)_{n \geq 1}$ be a sequence of independent, mean zero random vectors in $\R^d$ such that $\|\E\epsilon_n\epsilon_n^\top\|_{op} \leq \gamma$ for all $n \geq 1$. Let $\delta \in (0, 1)$ be an arbitrary confidence parameter. Then, with probability at least $1 - \delta$, simultaneously for all $n \geq 1$ we have
\begin{align*}
\|\Epsilon_n\|_{op}^2 \leq  \begin{cases} 
\beta\left(3\sqrt{n\ell_{\delta/2\cN}(n)} + 5\ell_{\delta/2\cN}(n)\right) + n\gamma \quad \text{when Assumption~\ref{ass:noise1} holds} \\
\frac{3}{2}\sqrt{nCd\gamma\ell_\delta(n)} + \frac{7}{3}Cd\ell_\delta(n) + n \gamma \quad \text{when Assumption~\ref{ass:noise2} holds}.
\end{cases}
\end{align*}
where $\ell_\delta$ is as defined in Lemma~\ref{lem:howard:mixture}, $\beta = 32\sigma^2 e^2$, and $\cN = 17^d$ is an upper bound on the $1/8$-covering number of $\S^{d- 1}$.
\end{lemma}

\begin{proof}
Observe the following basic chain of reasoning.
\[
\|\Epsilon_n\|_{op}^2 = \|\Epsilon_n^\top\Epsilon_n\|_{op} \leq \|\Epsilon_n^\top\Epsilon_n - \E\Epsilon_n^\top \Epsilon_n\|_{op} + \|\E\Epsilon_n^\top \Epsilon_n\|_{op}.
\]

Now, both Assumption~\ref{ass:noise1} and \ref{ass:noise2} imply that $\|\E[\epsilon_n\epsilon_n^\top]\|_{op} \leq \gamma$ for all $n \geq 1$. Consequently, noting that $\Epsilon_n^\top\Epsilon_n = \sum_{m = 1}^n \epsilon_m \epsilon_m^\top$, we have that
\[
\|\E\Epsilon_n^\top \Epsilon_n\|_{op} \leq \sum_{m = 1}^n \|\E\epsilon_m \epsilon_m^\top\|_{op} \leq n\gamma.
\]
Likewise, we have, applying Lemma~\ref{lem:noise:subgaussian} in the case the $\epsilon_n$ satisfy the subGaussian assumption and Lemma~\ref{lem:noise:bounded} in the case the noise satisfies the bounded assumption, that, with probability at least $1 - \delta$, simultaneously for all $n \geq 1$,
\[
\|\Epsilon_n^\top\Epsilon_n - \E\Epsilon_n^\top \Epsilon_n\|_{op} \leq \begin{cases}
\beta\left(3\sqrt{n\ell_{\delta/2\cN}(n)} + 5\ell_{\delta/2\cN}(n)\right) \quad \text{ when Assumption~\ref{ass:noise1} holds}\\

 \frac{3}{2}\sqrt{nCd\gamma\ell_\delta(n)} + \frac{7}{3}Cd\ell_\delta(n) \quad \text{ when Assumption~\ref{ass:noise2} holds}.
\end{cases}
\]
Adding these two lines of reasoning together yields the desired conclusion.

\end{proof}

We can now apply Lemma~\ref{lem:int:cov_noise} to bound the rate at which $\wh{\vP}_n,$ the projection operator onto the learned subspace at time $n$, converges to the projection operator $\vP$ onto the true, low-dimensional subspace.

\begin{lemma}[\textbf{Projection Convergence}]\label{lem:proj}
Let $\wh{\vP}_n$ denote the projection operator onto the learned subspace. Further, let $\vP$ denote the projection operator onto the true unknown subspace $W^\ast$. Assume $\rank(\vX_{n_0}) = r$ for some $n_0 \geq 1$. Then, for any $\delta \in (0, 1)$, we have with probability at least $1 - \delta$, simultaneously for all $n \geq n_0$,
\[
\|\wh{\vP}_n - \vP\|_{op}^2 \leq \begin{cases} \frac{4\beta\left(3\sqrt{n\ell_{\delta/2\cN}(n)} + 5\ell_{\delta/2\cN}(n)\right) + 4n\gamma}{\sigma_r(\vX_n)^2} \quad \text{when Assumption~\ref{ass:noise1} holds}\\
\frac{6\sqrt{nCd\gamma\ell_\delta(n)} + \frac{14}{3}Cd\ell_\delta(n) + 4n \gamma }{\sigma_r(\vX_n)^2} \quad \text{when Assumption~\ref{ass:noise2} holds}.

\end{cases}
\]

\end{lemma}

\begin{proof}
We write the singular value decompositions of $\vZ_n$ and $\vX_n$ respectively as
\[
\vZ_n = \wh{\vU}_n\wh{\Sigma}_n\wh{\vV}_n^\top \quad \text{and} \quad \vX_n = \wc{\vU}_n\wc{\Sigma}_n\wc{\vV}_n^\top.
\]
Since we have assumed $\rank(\vX_{n_0}) = r$ for some $n_0 \geq 1$, we have, for all $n \geq n_0$,
$$
\wh{\vP}_n = \wh{\vV}_{n, r}\wh{\vV}_{n, r}^T \quad \text{and} \quad \vP = \wc{\vV}_{n, r}\wc{\vV}_{n, r}^T.
$$
Now, applying Lemma~\ref{lem:ext:wedin} and Lemma~\ref{lem:int:cov_noise} we have that, with probability at least $1 - \delta$, simultaneously for all $n \geq n_0$,
\begin{align*}
\left\|\wh{\vP}_n - \vP\right\|_{op}^2 &= \left\|\wh{\vV}_{n, r}\wh{\vV}_{n, r}^T  - \wc{\vV}_{n, r}\wc{\vV}_{n, r}^T\right\|_{op}^2 \\
&\leq \frac{4\left\|\vX_n - \vZ_n\right\|_{op}^2}{\sigma_r(\vX_n)^2}  \\
&= \frac{4\|\Epsilon_n\|_{op}^2}{\sigma_r(\vX_n)^2} \\
&\leq \begin{cases} \frac{4\beta\left(3\sqrt{n\ell_{\delta/2\cN}(n)} + 5\ell_{\delta/2\cN}(n)\right) + 4n\gamma}{\sigma_r(\vX_n)^2} \quad \text{when Assumption~\ref{ass:noise1} holds}\\
\frac{6\sqrt{nCd\gamma\ell_\delta(n)} + \frac{14}{3}Cd\ell_\delta(n) + 4n \gamma }{\sigma_r(\vX_n)^2} \quad \text{when Assumption~\ref{ass:noise2} holds}.
\end{cases}
\end{align*}
This proves the desired result.
\end{proof}

We can obtain the following empirical version of Lemma~\ref{lem:proj}. The proof of the following is identical---the only difference is that, in the application of Wedin's theorem (Lemma~\ref{lem:ext:wedin}), we put the singular values of $\vZ_n$ is the denominator instead of $\vX_n$. This inequality is, in practice, more useful in computing confidence bounds, as the singular values of $\vZ_n$ are computable, whereas the singular values of $\vX_n$ are not.

\begin{lemma}[\textbf{Projection Convergence}]\label{lem:proj_emp}
Let $\wh{\vP}_n$ denote the projection operator onto the learned subspace. Further, let $\vP$ denote the projection operator onto the true unknown subspace $W^\ast$. Assume $\rank(\vX_{n_0}) = r$ for some $n_0 \geq 1$. Then, for any $\delta \in (0, 1)$, we have with probability at least $1 - \delta$, simultaneously for all $n \geq n_0$,
\[
\|\wh{\vP}_n - \vP\|_{op} \leq \begin{cases} \frac{4\beta\left(3\sqrt{n\ell_{\delta/2\cN}(n)} + 5\ell_{\delta/2\cN}(n)\right) + 4n\gamma}{(\sigma_r(\vZ_n) - \sigma_{r + 1}(\vZ_n))^2} \quad \text{when Assumption~\ref{ass:noise1} holds}\\
\frac{6\sqrt{nCd\gamma\ell_\delta(n)} + \frac{14}{3}Cd\ell_\delta(n) + 2n \gamma }{(\sigma_r(\vZ_n) - \sigma_{r + 1}(\vZ_n))^2} \quad \text{when Assumption~\ref{ass:noise2} holds}.

\end{cases}
\]

\end{lemma}

What remains is to provide a concentration inequality bounding the deviations between the singular values of the noisy covariate matrix $\vZ_n$ and the true, low rank covariate matrix $\vX_n$. We have the following inequality.

\begin{lemma}
\label{lem:int:sv_conc}
Let $\delta \in (0, 1)$ be arbitrary. Then with probability at least $1 - \delta$ we have, simultaneously for all $i \in [r]$ and $n \geq 1$,
\[
\left|\sigma_i(\vZ_n) - \sigma_i(\vX_n)\right|^2 \leq \begin{cases} 
\beta\left(3\sqrt{n\ell_{\delta/2\cN}(n)} + 5\ell_{\delta/2\cN}(n)\right) + n\gamma \quad \text{when Assumption~\ref{ass:noise1} holds} \\
\frac{3}{2}\sqrt{nCd\gamma\ell_\delta(n)} + \frac{7}{3}Cd\ell_\delta(n) + n \gamma \quad \text{when Assumption~\ref{ass:noise2} holds}.

\end{cases}
\]
In addition, on the same probability at least $1 - \delta$ event, for all $n \geq 1$ such that $\snr_n \geq 2$, we have
\[
\frac{\sigma_r(\vX_n)}{2} \leq \sigma_r(\vZ_n) \leq \frac{3\sigma_r(\vX_n)}{2}.
\]
Note that the same holds if $\snr_n$ is replaced with the action-specific signal to noise ratio $\snr_n(a)$.

\end{lemma}

\begin{proof}
By Lemma~\ref{lem:ext:weyl}, we know that for any $i \in [r]$, 
$$
\left|\sigma_i(\vZ_n) - \sigma_i(\vX_n)\right| \leq \left\|\vZ_n - \vX_n\right\|_{op} = \left\|\Epsilon_n\right\|_{op}.
$$
Thus, applying Lemma~\ref{lem:int:cov_noise} yields the first part of the theorem.

Now, suppose $\snr_n \geq 2$. Let $U_n^2$ denote either the first or second line of the already-shown inequality (depending on whether Assumption~\ref{ass:noise1} or Assumption~\ref{ass:noise2} holds). Then with probability at least $1 - \delta$,
\begin{align*}
\sigma_i(\vZ_n) &\geq \sigma_i(\vX_n) - U_n \\
&= \sigma_i(\vX_n) - \sigma_i(\vX_n)\frac{U_n}{\sigma_i(\vX_n)}\\
&\geq \sigma_i(\vX_n) - \sigma_i(\vX_n)\frac{U_n}{\sigma_r(\vX_n)}\\
&= \sigma_i(\vX_n) - \sigma_i(\vX_n)\frac{1}{\snr_n}\\
&\geq \frac{\sigma_i(\vX_n)}{2}.
\end{align*}
\end{proof}

\subsection{Convergence Results for Regression in the True Subspace}
\label{app:reg}

In this section, we construct confidence ellipsoids bounding the error between $\wc{\theta}_n(a)$ and $\theta(a)$, where
$$
\wc{\theta}_n(a) := \wc{\cV}_n(a)^{-1}\wc{\vZ}_n(a)^\top\vY_n(a),
$$
i.e. $\wc{\theta}_n(a)$ is the estimate of $\theta(a)$ given access to the true, underlying subspace $W^\ast$. While the learner never has direct access to $W^\ast$, the quantity $\wc{\theta}_n(a)$ proves useful in bounding $\|\wh{\theta}_n(a) - \theta(a)\|_2^2$, the $\ell_2$ error between the PCR estimate of $\theta(a)$ and $\theta(a)$ itself. In other words, the bounds proved in this section provide a theoretical tool for understanding the convergence of PCR in adaptive settings. The main lemma in this appendix is the following.

\begin{lemma}[\textbf{Error bound in true subspace}]
\label{lem:ellipsoid}
Let $\delta \in (0, 1)$ be an arbitrary confidence parameter. Then, with probability at least $1 - 2\delta$, simultaneously for all $a \in [A]$ and $n \geq 1$, we have 
\begin{align*}
&\left\|\wc{\cV}_n(a)^{1/2}\left(\wc{\theta}_n(a) - \theta(a)\right)\right\|_2^2 \leq 4\rho L^2 + 8\eta^2\left[\log\left(\frac{A}{\delta}\right) +  r\log\left(1 + \frac{\sigma_1(\vZ_n(a))^2}{\rho}\right)\right]\\
&+  2L^2\begin{cases} 
\beta\left(3\sqrt{n\ell_{\delta/2\cN}(n)} + 5\ell_{\delta/2\cN}(n)\right) + n\gamma \quad \text{when Assumption~\ref{ass:noise1} holds} \\
\frac{3}{2}\sqrt{nCd\gamma\ell_\delta(n)} + \frac{7}{3}Cd\ell_\delta(n) + n \gamma \quad \text{when Assumption~\ref{ass:noise2} holds}.
\end{cases}
\end{align*}
where $\beta$, $\ell_\delta$, and $\cN$ are as defined in Lemma~\ref{lem:noise:subgaussian} and \ref{lem:noise:bounded}.
\end{lemma}

To make the proof of the above lemma more modular, we introduce several lemmas. The first of these lemmas simply provides an alternative, easier-to-bound representation of the difference (or error) vector $\wc{\theta}_n(a) - \theta(a)$

\begin{lemma}
\label{lem:int:diff}
For any $n \in [N]$ and $a \in [A],$ we have
\[
\wc{\theta}_n(a) - \theta(a) = \wc{\cV}_n(a)^{-1}\wc{\vZ}_n(a)^\top \Xi_n(a) + \wc{\cV}_n(a)^{-1}\wc{\vZ}_n(a)^\top \wc{\Epsilon}_n(a)\theta(a) - \rho\wc{\cV}_n(a)^{-1}\theta(a).
\]
\end{lemma}

\begin{proof}
A straightforward computation yields
\begin{align*}
\wc{\theta}_n(a) - \theta(a) &= \wc{\cV}_n(a)^{-1}\wc{\vZ}_n(a)^\top \vY_n(a) - \theta(a) \\
&= \wc{\cV}_n(a)^{-1}\wc{\vZ}_n(a)(\vX_n(a)\theta(a) + \Xi_n(a)) - \theta(a) \\
&=  \wc{\cV}_n(a)^{-1}\wc{\vZ}_n(a)\vX_n(a)\theta(a) + \wc{\cV}_n(a)^{-1}\wc{\vZ}_n(a)^\top \Xi_n(a) -\theta(a) \\
&\qquad \pm \wc{\cV}_n(a)^{-1}\wc{\vZ}_n(a)\wc{\Epsilon}_n(a)\theta(a) \pm \rho\wc{\cV}_n(a)^{-1}\theta(a)\\
&= \wc{\cV}_n(a)^{-1}\left[\wc{\vZ}_n(a)^\top (\vX_n(a) + \wc{\Epsilon}_n(a)) + \rho I_d\right]\theta(a) + \wc{\cV}_n(a)^{-1}\wc{\vZ}_n(a)^\top \Xi_n(a) - \theta(a) \\
&\qquad - \wc{\cV}_n(a)^{-1}\wc{\vZ}_n(a)\wc{\Epsilon}_n(a)\theta(a) - \rho\wc{\cV}_n(a)^{-1}\theta(a)\\
&= \wc{\cV}_n(a)^{-1}\left[\wc{\vZ}_n(a)^\top \wc{\vZ}_n(a) + \rho I_d\right]\theta(a) + \wc{\cV}_n(a)^{-1}\wc{\vZ}_n(a)^\top \Xi_n(a) - \theta(a) \\
&\qquad - \wc{\cV}_n(a)^{-1}\wc{\vZ}_n(a)\wc{\Epsilon}_n(a)\theta(a) - \rho\wc{\cV}_n(a)^{-1}\theta(a)\\
&= \wc{\cV}_n(a)^{-1}\wc{\vZ}_n(a)^\top \Xi_n(a) - \wc{\cV}_n(a)^{-1}\wc{\vZ}_n(a)^\top \wc{\Epsilon}_n(a)\theta(a) - \rho\wc{\cV}_n(a)^{-1}\theta(a),
\end{align*}
which proves the desired result.
\end{proof}

We leverage the following technical lemma in bounding the determinant of $\wc{\cV}_n(a)$, the projection of the covariance matrix $\cV_n(a)$ onto the true, unknown subspace $W^\ast$.

\begin{lemma}[\textbf{Determinant-Trace Inequality}]
\label{lem:ext:det_trace}
Let $\vA \in \R^{d \times d}$ be a positive-semidefinite matrix of rank $r$. Then, for any $\rho > 0$, we have
$$
\log\det(I_d + \rho^{-1}\cA) \leq r\log\left(1 + \frac{\sigma_1(\vA)}{\rho}\right).
$$

\end{lemma}

\begin{proof}
We have
\begin{align*}
\log\det(I_d + \rho^{-1}\vA) &= r\log\left[\det(I_d + \rho^{-1}\vA)^{1/r}\right] \\
& = r\log\prod_{i = 1}^r\left(1 + \rho^{-1}\sigma_i(\vA)\right)^{1/r} \\
&\leq r\log\left(\sum_{i = 1}^r\frac{1 + \rho^{-1}\sigma_i(\vA)}{r}\right) \\
&\leq r\log\left(1 + \frac{\sigma_1(\vA)}{\rho}\right).
\end{align*}

\end{proof}
\begin{lemma}
\label{lem:int:term_bd}
Let $\vA \in \R^{t \times d}$ be a matrix, and let $\rho > 0$ be arbitrary. Suppose $\vA$ has $1 \leq k \leq t \land d$ non-zero singular values. Then,
$$
\left\|(\vA^T\vA + \rho I_d)^{-1/2}\vA^T\right\|_{op} \leq \frac{\sigma_k(\vA)}{\sqrt{\sigma_k(\vA)^2 +\rho}} \leq 1
$$

\end{lemma}

\begin{proof}
    Let us write the singular value decomposition of $\vA$ as
    $$
    \vA = \vU \Sigma \vV^T = \vU_k \Sigma_k \vV_k^T,
    $$
    where the second equality follows from the fact that $\vA$ has exactly $k$ non-zero singular values. We have the equality
    $$
    \vA^T\vA + \rho I_d = \vV (\Sigma_k^2 + \rho I_k) \vV^T.
    $$
    Thus, we see
    \begin{align*}
    \left\|(\vA^T\vA + \rho I_d)^{-1/2}\vA^T\right\|_{op} &= \left\|\vV_k (\Sigma_k^2 + \rho I_k)^{-1/2}\vV_k^T \vV_k \Sigma_k \vU_k^T\right\|_{op} \\
    &= \left\|(\Sigma_k^2 + \rho I_k)^{-1/2}\Sigma_k\right\|_{op} \\
    &= \max_{i = 1}^k \frac{\sigma_i(\vA)}{\sqrt{\rho + \sigma_i(\vA)^2}} \leq \frac{\sigma_k(\vA)}{\sqrt{\rho + \sigma_k(\vA)^2}}.
    \end{align*}
\end{proof}

With the aforementioned technical lemmas, alongside the time-uniform martingale bounds presented in Appendix~\ref{app:martingale} and the singular value concentration results proved in Appendix~\ref{app:pca}, we can now prove Lemma~\ref{lem:ellipsoid}.

\begin{proof}[\textbf{Proof of Lemma~\ref{lem:ellipsoid}}]
First, by Lemma~\ref{lem:int:diff}, we see that we have, for any $a \in [A]$ and $n \geq 1$,
\[
\wc{\theta}_n(a) - \theta(a) = \wc{\cV}_n(a)^{-1}\wc{\vZ}_n(a)^\top\Xi_n(a) + \wc{\cV}_n(a)^{-1}\wc{\vZ}_n(a)^\top\wc{\Epsilon}_n(a)\theta(a) - \rho\wc{\cV}_n(a)^{-1}\theta(a).
\]
With this decomposition in hand, we can apply the parallelogram inequality ($\|x - y\|^2_2 + \|x + y\|^2_2 \leq 2\|x\|_2^2 + 2\|y\|_2^2$) twice to see that
\begin{align*}
&\left\|\wc{\cV}_n(a)^{1/2}\left(\wc{\theta}_n(a) - \theta(a)\right)\right\|_2^2 \\
&\leq 4\left\|\wc{\cV}_t(a)^{-1/2}\wc{\vZ}_n(a)^\top\Xi_n(a)\right\|_2^2 + 4\rho^2\left\|\wc{\cV}_n(a)^{-1/2}\theta(a)\right\|_2^2 + 2\left\|\wc{\cV}_n(a)^{-1/2}\wc{\vZ}_n(a)^\top\wc{\Epsilon}_n(a)\theta(a)\right\|_2^2 \\
&\leq 4\left\|\wc{\cV}_n(a)^{-1/2}\wc{\vZ}_n(a)^\top\Xi_n(a)\right\|_2^2 + 4\rho L^2 + 2\left\|\wc{\cV}_n(a)^{-1/2}\wc{\vZ}_n(a)^\top\wc{\Epsilon}_n(a)\theta(a)\right\|_2^2.
\end{align*}

We bound the first and last terms separately. Applying Lemma~\ref{lem:ext:mixture} for each $a \in [A]$ and taking a union bound over actions yields that, with probability at least $1 - \delta$, simultaneously for all $n \geq 1$ and $a \in [A]$,
\begin{align*}
\left\|\wc{\cV}_n(a)^{-1/2}\wc{\vZ}_n(a)^\top\Xi_n(a)\right\|_2 &\leq \eta\sqrt{2\log\left(\frac{A}{\delta}\sqrt{\det\left(\rho^{-1}\wc{\cV}_n(a)\right)}\right)} \\
&\leq \eta\sqrt{2\left[\log\left(\frac{A}{\delta}\right) +  r\log\left(1 + \frac{\sigma_1(\wc{\vZ}_n(a))^2}{\rho}\right)\right]} \\
&\leq \eta\sqrt{2\left[\log\left(\frac{A}{\delta}\right) +  r\log\left(1 + \frac{\sigma_1(\wh{\vZ}_n(a))^2}{\rho}\right)\right]},
\end{align*}
where the second inequality comes from applying Lemma~\ref{lem:ext:det_trace} and the third inequality comes from the fact that $\sigma_i(\wh{\vZ}_n(a)) \geq \sigma_i(\wc{\vZ}_n(a))$ for all $i \in [r]$.

Next, we bound the final term in the above expansion. Observe that
\begin{align*}
\left\|\wc{\cV}_n(a)^{-1/2}\wc{\vZ}_n(a)^\top\wc{\Epsilon}_n(a)\theta(a)\right\|_2 &\leq \left\|\wc{\cV}_n(a)^{-1/2}\wc{\vZ}_n(a)^\top\right\|_{op}\left\|\wc{\Epsilon}_n(a)\right\|_{op}\left\|\theta(a)\right\|_2.
\end{align*}
First, note that $\|\theta(a)\|_2 \leq L$ by assumption. Next, Lemma~\ref{lem:int:term_bd} yields
\[
\left\|\wc{\cV}_n(a)^{-1/2}\wc{Z}_n(a)^\top\right\|_{op} \leq 1.
\]
Lastly, note that by Lemma~\ref{lem:int:cov_noise}, we have that, with probability at least $1 - \delta$, simultaneously for all $n \geq 1$
\begin{align*}
\|\wc{\Epsilon}_n(a)\|_{op}^2 &\leq \|\Epsilon_n\|^2
&\leq \begin{cases} 
\beta\left(3\sqrt{n\ell_{\delta/2\cN}(n)} + 5\ell_{\delta/2\cN}(n)\right) + n\gamma \quad \text{when Assumption~\ref{ass:noise1} holds} \\
\frac{3}{2}\sqrt{nCd\gamma\ell_\delta(n)} + \frac{7}{3}Cd\ell_\delta(n) + n \gamma \quad \text{when Assumption~\ref{ass:noise2} holds}.
\end{cases}
\end{align*}
This proves the desired inequality.
\end{proof}

\subsection{Technical Lemmas for~\Cref{thm:emp_bd}}
\label{app:lem_main}

In this appendix, we present several additional technical lemmas that are needed for bounding terms in the main theorem.

\begin{lemma}
\label{lem:t1}
Assuming the setup of Theorem~\ref{thm:emp_bd}, we have, for all $n \geq 1$,
\[
\left\|\wh{\vZ}_n(a)\wh{\theta}_n(a) - \vX_{n}(a)\theta(a)\right\|_2^2 \leq 8\left\|\wc{\cV}_n(a)^{1/2}\left(\wc{\theta}_n(a) - \theta(a)\right)\right\|_2^2 + 6 \left\|\Xi_n(a)\right\|_2^2 + 8\left\|\wh{\vZ}_n(a)\theta(a) - \vX_n(a)\theta(a)\right\|_2^2 
\]
\end{lemma}

\begin{proof}
A relatively straightforward computation yields
\begin{align*}
\left\|\wh{\vZ}_n(a)\wh{\theta}_n(a) - \vX_{t}(a)\theta(a)\right\|_2^2 &= \left\|\wh{\vZ}_n(a)\wh{\theta}_n(a) - \vX_n(a)\theta(a) + \Xi_n(a) - \Xi_n(a)\right\|_2^2\\
&\leq 2\left\|\wh{\vZ}_n(a)\wh{\theta}_n(a) - \vY_n(a)\right\|_2^2 + 2\left\|\Xi_n(a)\right\|_2^2 \\
&\leq 2\left\|\wh{\vZ}_n(a)\wc{\theta}_n(a) - \vY_n(a)\right\|_2^2 + 2\left\|\Xi_n(a)\right\|_2^2 \\
&\leq 4\left\|\wh{\vZ}_n(a)\wc{\theta}_n(a) - \vX_n(a)\theta(a)\right\|_2^2 + 6\left\|\Xi_n(a)\right\|^2_2 \\
&\leq 8\left\|\wh{\vZ}_n(a)\left(\wc{\theta}_n(a) - \theta(a)\right)\right\|_2^2 + 6\left\|\Xi_n(a)\right\|_2^2 + 8\left\|\wh{\vZ}_n(a)\theta(a) - \vX_n(a)\theta(a)\right\|_2^2 \\
&= 8\left\|\wh{\vZ}_n(a)\vP\left(\wc{\theta}_n(a) - \theta(a)\right)\right\|_2^2 + 6\left\|\Xi_n(a)\right\|_2^2 + 8\left\|\wh{\vZ}_n(a)\theta(a) - \vX_n(a)\theta(a)\right\|_2^2\\
&\leq 8\left\|\wc{\vZ}_n(a)\left(\wc{\theta}_n(a) - \theta(a)\right)\right\|_2^2 + 6\left\|\Xi_n(a)\right\|_2^2 + 8\left\|\wh{\vZ}_n(a)\theta(a) - \vX_n(a)\theta(a)\right\|_2^2\\
&\leq 8\left\|\wc{\cV}_n(a)^{1/2}\left(\wc{\theta}_n(a) - \theta(a)\right)\right\|_2^2 + 6 \left\|\Xi_n(a)\right\|_2^2 + 8\left\|\wh{\vZ}_n(a)\theta(a) - \vX_n(a)\theta(a)\right\|_2^2\\
\end{align*}
where we apply the the parallelogram inequality to obtain the first and third inequalities. 
The second inequality follows from the characterization of ridge regression in Fact~\ref{fact:ridge_const} and the fact $\rho$ is chosen sufficiently small. 
\end{proof}

\begin{lemma}
\label{lem:t2}
Assuming the setup of Theorem~\ref{thm:emp_bd}, we have, for all $n \geq 1$,
\[
\left\|\vX_n(a) \theta(a) - \wh{\vZ}_n(a)\theta(a)\right\|_2^2 \leq 2L^2\sigma_1(\vZ_n(a))^2\left\|\vP - \wh{\vP}_n(a)\right\|_{op}^2 + 2L^2\left\|\Epsilon_n\right\|_{op}^2
\]

\end{lemma}

\begin{proof}
A relatively simple computation yields
\begin{align*}
\left\|\vX_n(a) \theta(a) - \wh{\vZ}_n(a)\theta(a)\right\|_2^2 &= \left\|\vX_n(a)\theta(a) - \wc{\vZ}_n(a)\theta(a) + \wc{\vZ}_n(a)\theta(a) - \wh{\vZ}_n(a)\theta(a)\right\|_2^2\\
&\leq  2\left\|\wc{\vZ}_n(a)\theta(a) - \wh{\vZ}_n(a)\theta(a)\right\|_2^2 + 2\left\|\wc{\vZ}_n(a)\theta(a) - \vX_n(a)\theta(a)\right\|_2^2\\
&= 2\left\|\wc{\vZ}_n(a)\theta(a) - \wh{\vZ}_n(a)\theta(a)\right\|_2^2 + 2\left\|\wc{\Epsilon}_n(a)\theta(a)\right\|_2^2 \\
&= 2\left\|\vZ_n(a)(\vP - \wh{\vP}_n(a))\theta(a)\right\|_2^2 + 2\left\|\wc{\Epsilon}_n(a)\theta(a)\right\|_2^2 \\
&\leq 2L^2\sigma_1(\vZ_n(a))^2\left\|\vP - \wh{\vP}_n(a)\right\|_{op}^2 + 2L^2\left\|\wc{\Epsilon}_n\right\|_{op}^2,
\end{align*}
where we have applied the parallelogram inequality to obtain the first inequality.
\end{proof}

\begin{lemma}
\label{lem:resp_bd}
Suppose $(\xi_n)_{n \geq 1}$ is a sequence of independent, $\eta$-subGaussian random variables. Then, for any $\delta \in (0, 1)$, we have simultaneously for all $n \geq 1$
\[
\|\Xi_n\|_2^2 \leq 6\eta^2\sqrt{2n\ell_\delta(n)} + 10\eta^2\ell_\delta(n) + n\alpha,
\]
where we recall $\Xi_n := (\xi_1, \dots, \xi_n)^\top$.
\end{lemma}
\begin{proof}
First, observe that by the appendix of \citet{honorio2014tight}, since $\xi_n$ is $\eta$-subGaussian, $\xi_n^2 - \E \xi_n^2$ is $(4\sqrt{2}\eta^2, 4\eta^2)$-subExponential, and thus by our discussion in Appendix~\ref{app:martingale} $\xi_n^2 - \E \xi_n^2$ is also $(4\sqrt{2}\eta^2, 4\eta^2)$-subGamma. Further, by Assumption~\ref{ass:noise_resp}, we have $\sum_{m = 1}^n \E\xi_m^2 \leq n\alpha$. Recall that Lemma~\ref{lem:howard:mixture} yields that, with probability at least $1 - \delta$,
\[
\sum_{m = 1}^n \xi_m^2 - \E\xi_m^2 \leq 6\eta^2\sqrt{2n\ell_\delta(n)} + 10\eta^2\ell_\delta(n).
\]
Piecing everything together, we have that
\begin{align*}
\|\Xi_n\|_2^2 &= \sum_{m = 1}^n \xi_m^2\\
&= \sum_{m = 1}\xi_m^2 - \E\xi_m^2  + \sum_{m = 1}^n\E\xi_m^2 \\
&\leq 6\eta^2\sqrt{2n\ell_\delta(n)} + 10\eta^2\ell_\delta(n) + n\alpha.
\end{align*}
\end{proof}
We now mention the finishing steps in proving Theorem~\ref{thm:emp_bd}, carrying over from where the proof in the paper ended. We see, with the previously addressed inequalities that, with probability at least $1 - O(A\delta)$, simultaneously for all $n \geq n_0$ and $a \in [A]$,
\begin{align*}
&\left\|\wh{\theta}_n(a) - \theta(a) \right\|_2^2 \leq \frac{6L^2U_n^2}{\sigma_r(\vZ_n(a))^2} + \frac{2}{\sigma_r(\vZ_n(a))^2}(T_1 + T_2) \\
&= \frac{6L^2U_n^2}{\sigma_r(\vZ_n(a))^2} + \frac{2}{\sigma_r(\vZ_n(a))^2}\Bigg[32\rho L^2 + 64\eta^2\left(\log\left(\frac{A}{\delta}\right) +  r\log\left(1 + \frac{\sigma_1(\vZ_n)^2}{\rho}\right)\right)\\
&+ 16L^2U_n^2 + 6\eta^2\sqrt{2c_n(a)\ell_\delta(c_n(a))} + 10\eta^2\ell_\delta(c_n(a)) + 6c_n(a)\alpha + +\frac{108L^2\sigma_1(\vZ_n(a))^2U_n^2}{\sigma_r(\vZ_n(a))^2} + 18L^2U_n^2\Bigg] \\
&= \frac{6L^2}{\wh{\snr}_n(a)^2} + \frac{L^2}{\wh{\snr}_n(a)^2}\left[68 + \frac{216\sigma_1(\vZ_n(a))^2}{\sigma_r(\vZ_n(a))^2}\right] + \frac{2}{\sigma_r(\vZ_n(a))^2}\Bigg[32\rho L^2 \\
&\quad +64\eta^2\left(\log\left(\frac{A}{\delta}\right)
+ r\log\left(1 + \frac{\sigma_1(\vZ_n(a))^2}{\rho}\right)\right)
 + 6\eta^2\sqrt{2c_n(a)\ell_\delta(c_n(a))} + 10\eta^2\ell_\delta(c_n(a)) + 6c_n(a)\alpha\Bigg]
\end{align*}
\subsection{Proof of~\Cref{thm:nice_bd}}
\label{app:nice}

In this section, we prove~\Cref{thm:nice_bd}. All that is required in proving this bound is simplifying the results of~\Cref{thm:emp_bd}, given we allow ourselves slack to control the bound up to universal constants and poly-logarithmic factors.

\begin{proof}[\textbf{Proof of~\Cref{thm:nice_bd}}]

Recall that by~\Cref{thm:emp_bd}, we have with probability at least $1 - O(A\delta)$, simultaneously for all $n \geq n_0$ 
\begin{align*}
&\left\|\wh{\theta}_n(a) - \theta(a) \right\|_2^2 \leq \frac{L^2}{\wh{\snr}_n(a)^2}\left[74 + 216\kappa(\vZ_n(a))^2\right] + \frac{2\err_n(a)}{\sigma_r(\vZ_n(a))^2}.
\end{align*}
where in the above we define the ``error'' term $\err_n(a)$ to be
\begin{equation*}
    \begin{aligned}
        \err_n(a) &:= \underbrace{24\rho L^2 +64\eta^2 \left(\log\left(\frac{A}{\delta}\right) + r\log\left(1 + \frac{\sigma_1(\vZ_n(a))^2}{\rho}\right)\right)}_{T_1}\\
        &+ \underbrace{6\eta^2\sqrt{2c_n(a)\ell_\delta(c_n(a))} + 10\eta^2\ell_\delta(c_n(a)) + 6c_n(a)\alpha}_{T_2}.
    \end{aligned}
\end{equation*}

First, by the second part of Lemma~\ref{lem:int:sv_conc} coupled with the assumption that $\snr_n(a) \geq 2$ for all $n \geq n_0$ and $a \in [A]$, we have that $\snr_n(a) = \Theta(\wh{\snr}_n(a))$. Further, by this same result, we have that $\kappa(\vZ_n(a)) = \Theta(\kappa(\vX_n(a)))$.
From this, it is clear that $\frac{L^2}{\wh{\snr}_n(a)^2}\left[74 + 216\kappa(\vZ_n(a))^2\right] = O\left(\frac{1}{\snr_n(a)^2}\kappa(\vX_n(a))^2\right)$. What remains is to show that $\frac{2\err_n(a)}{\sigma_r(\wh{\vZ}_n(a))^2} = \wt{O}\left(\frac{1}{\snr_n(a)^2}\right)$. 
To this end, it suffices to show that $\err_n(a) = \wt{O}(n + d)$. But this is trivial, as it is clear that $T_1 = \wt{O}(1)$ and $T_2 = \wt{O}(c_n(a) ) = \wt{O}(n + d)$. 
Thus, we have proved the desired result.

\end{proof}
\section{Appendix for~\Cref{sec:panel}: Applications to Causal Inference with Panel Data}\label{app:panel}
\subsection{Proofs for~\Cref{sec:sc}: Adaptive Synthetic Control}
\paragraph{Notation}
The following table summarizes the changes in notation between the setting of~\Cref{sec:oPCR} and that of adaptive synthetic control. 
\begin{center}
\begin{tabular}{||c | c||} 
 \hline
 Notation of~\Cref{sec:oPCR} & Notation of Adaptive Synthetic Control \\ [0.5ex] 
 \hline\hline
 $n$ & $T_0$ \\ 
 \hline
 $d$ & $c_n(a)$ \\
 \hline
 $X_n$ & $X_t = \E Y_{c_n(a), t}^{(0)} \in \mathbb{R}^{c_n(a)}$ \\
 \hline
 $Z_n$ & $Z_t = \E Y_{c_n(a), t}^{(0)} \in \mathbb{R}^{c_n(a)}$ \\
 \hline 
 $\epsilon_n$ & $\epsilon_t \in \mathbb{R}^{c_n(a)}$ \\
 \hline
  $Y_n$ & $Y_t = Y_{n, t}^{(0)} = \langle \theta_n(a), \E Y_{c_n(a), t}^{(0)} \rangle + \epsilon_{n,t}$ \\
 \hline
  $\vP \in \mathbb{R}^{d \times d}$ & $\vP \in \mathbb{R}^{c_n(a) \times c_n(a)}$ \\
 \hline
  $\wh \vP_n(a)$ & $\wh \vP_t(a) \in \mathbb{R}^{c_n(a) \times c_n(a)}$ \\
 \hline
  $\cN = 17^d$ & $\cN = 17^{c_n(a)}$ \\
 \hline
  $\eta$ & $\sigma$ \\
 \hline
  $\alpha$ & $\sigma^2$ \\
 \hline
  $\gamma$ & $\sigma^2$ \\
 \hline
  $\Epsilon_n$ & $\Epsilon_{T_0} = [\epsilon_{1, pre} \; \epsilon_{2, pre} \cdots \epsilon_{n, pre}] \in \mathbb{R}^{T_0 \times n}$ \\
 \hline
  $\Xi_n$ & $\Xi_{T_0} = [\epsilon_{n, 1}, \ldots, \epsilon_{n, T_0}]^{\top} \in \mathbb{R}^{T_0}$ \\
 \hline
 $\ell_{\delta}(n)$ & $\ell_{\delta}(T_0) := 2 \log\log(2 T_0) + \log \left( \frac{n \pi^2}{12 \delta} \right)$ \\ [1ex] 
 \hline
\end{tabular}
\end{center}

\begin{theorem}[\textbf{Prediction error; regularized synthetic interventions}]\label{thm:si-app}
    Let $\delta \in (0, 1)$ be an arbitrary confidence parameter and $\rho > 0$ be chosen to be sufficiently small, as detailed in Appendix~\ref{app:ridge}. 
    Further, assume that Assumptions \ref{ass:lfm} and \ref{ass:hlsi} are satisfied, there is some $n_0 \leq n$ such that $\rank(\vX_{n_0}(a)) = r$. 
    Under~\Cref{ass:spectrum} with probability at least $1 - O(A\delta)$, simultaneously for all interventions $a \in \{0, \ldots, A-1\}$
    \begin{equation*}
    \begin{aligned}
        |\widehat{\E} \pv{\Bar{Y}}{a}_{n,post} - \E \pv{\Bar{Y}}{a}_{n,post}| &= \frac{\sqrt{c_n(a)}}{\wh \snr_{T_0}(a)} \left( 3 \| \wh \theta_{T_0}(a) - \theta(a) \|_2 + \sqrt{34} L + \sqrt{108} \kappa(\vZ_{T_0}(a)) \right)\\ &+ \frac{\sqrt{c_n(a) \err_{T_0}(a)}}{\sigma_r(\vZ_{T_0}(a))} + \sigma \sqrt{\frac{\log(A / \delta)}{T - T_0}} \left( \| \wh \theta_{T_0}(a) - \theta(a) \|_2 + L \right)
    \end{aligned}
    \end{equation*}
    where $\widehat{\E}\pv{\Bar{Y}}{a}_{n,post}$ is the estimated average post-intervention outcome for unit $n$ under intervention $a$ given by the regularized synthetic interventions estimator (\Cref{def:si}). 
\end{theorem}
\begin{proof}
    \begin{equation*}
    \begin{aligned}
        \widehat{\E}\Bar{Y}_{n,post}^{(a)} - \E \Bar{Y}_{n,post}^{(a)} &= \frac{1}{T - T_0} \sum_{t=T_0 + 1}^T \langle \widehat{\theta}_n(a), Y_{n(a), t}^{(a)} \rangle - \langle \theta_n(a), \E Y_{n(a), t}^{(a)} \rangle\\
        &= \underbrace{\frac{1}{T - T_0} \sum_{t=T_0+1}^T \langle \widehat{\theta}_n(a) - \theta_n(a), \E Y_{n(a), t}^{(a)} \rangle}_{T_1}\\ 
        &+ \underbrace{\frac{1}{T - T_0} \sum_{t=T_0+1}^T \langle \theta_n(a), \epsilon_{n(a), t} \rangle}_{T_2} + \underbrace{\frac{1}{T - T_0} \sum_{t=T_0+1}^T \langle \widehat{\theta}_n(a) - \theta_n(a), \epsilon_{n(a), t} \rangle}_{T_3}
    \end{aligned}
    \end{equation*}

    \textbf{Bounding $T_1$:}
    \begin{equation*}
    \begin{aligned}
        T_1 &= \frac{1}{T - T_0} \sum_{t=T_0 + 1}^T \langle \widehat{\theta}_n(a) - \theta_n(a), \E Y_{n(a), t}^{(a)} \mathbf{P} \rangle\\
        &= \frac{1}{T - T_0} \sum_{t=T_0 + 1}^T \langle \mathbf{P}(\widehat{\theta}_n(a) - \theta_n(a)), \E Y_{n(a), t}^{(a)}  \rangle\\
        &\leq \frac{1}{T - T_0} \sum_{t=T_0 + 1}^T \| \mathbf{P}(\widehat{\theta}_n(a) - \theta_n(a)) \|_{2}  \|\E Y_{n(a), t}^{(a)} \|_2\\
        &\leq \sqrt{c_n(a)} \| \mathbf{P}(\widehat{\theta}_n(a) - \theta_n(a)) \|_{2}\\
        &\leq \underbrace{\sqrt{c_n(a)} \|\mathbf{P} - \widehat{\mathbf{P}}_n(a)\|_{op} \|\widehat{\theta}_n(a) - \theta_n(a)\|_2}_{T_{1.1}} + \underbrace{\sqrt{c_n(a)} \| \widehat{\mathbf{P}}_n(a) (\widehat{\theta}_n(a) - \theta_n(a)) \|_{2}}_{T_{1.2}}
    \end{aligned}
    \end{equation*}
    \paragraph{Bounding $T_{1.1}:$}
    \begin{equation*}
    \begin{aligned}
        \| \wh \vP_{T_0}(a) - \vP\|_{op} &\leq \frac{\sqrt{4\beta (3 \sqrt{T_0 \ell_{\delta / 2 \cN}(T_0)} + 5\ell_{\delta / 2 \cN}(T_0)) + 4T_0 \sigma^2}}{\sigma_r(\vX_{T_0})}\\
        &\leq \frac{3\sqrt{\beta (3 \sqrt{T_0 \ell_{\delta / 2 \cN}(T_0)} + 5\ell_{\delta / 2 \cN}(T_0)) + T_0 \sigma^2}}{\sigma_r(\vZ_{T_0})} = \frac{3}{\wh \snr_{T_0}(a)}
    \end{aligned}
    \end{equation*}
    Putting everything together, we get that 
    $$\sqrt{c_n(a)}\| \wh \vP_{T_0}(a) - \vP\|_{op} \|\wh \theta_{T_0}(a) - \theta(a)\|_2 \leq \frac{3 \sqrt{c_n(a)}}{\wh \snr_{T_0}(a)} \|\wh \theta_{T_0}(a) - \theta(a)\|_2.$$ 
    \paragraph{Bounding $T_{1.2}:$}
    \begin{equation*}
    \begin{aligned}
        \sqrt{c_n(a)} \| \widehat{\mathbf{P}}_n(a) (\widehat{\theta}_n(a) &- \theta_n(a)) \|_{2} \leq \frac{\sqrt{c_n(a)}}{\sigma_r(\vZ_{T_0}(a))} \| \vZ_{T_0}(a) (\wh \theta_{T_0}(a) - \theta(a)) \|_2\\
        &\leq \frac{2\sqrt{c_n(a)}}{\sigma_r(\vZ_{T_0}(a))} \left(  \underbrace{\| \vZ_{T_0}(a) \wh \theta_{T_0}(a) - \vX_{T_0}(a) \theta(a) \|_2^2}_{T_{1.2.1}} + \underbrace{\| \vX_{T_0}(a) \theta(a) - \vZ_{T_0}(a) \theta(a) \|_2^2}_{T_{1.2.2}} \right)^{1/2}
    \end{aligned}
    \end{equation*}
    \begin{equation*}
    \begin{aligned}
        T_{1.2.1} &\leq 32 \rho L^2 + 64 \sigma^2 (\log(A / \delta) + r \log (1 + \frac{\sigma_1(\vZ_{T_0}(a))^2}{\rho})) + 16L^2 U_{T_0}^2 + 6 \sigma^2 \sqrt{2T_0 \ell_{\delta}(T_0)}\\ &+ 10 \sigma^2 \ell_{\delta}(T_0) + 6 \sigma^2 T_0 + 8 T_{1.2.2}
    \end{aligned}
    \end{equation*}
    \begin{equation*}
        T_{1.2.2} \leq \frac{12 L^2 \sigma_1(\vZ_{T_0}(a))^2 U_{T_0}^2}{\sigma_r(\vZ_{T_0}(a))^2} + 2 L^2 U^2_{T_0}
    \end{equation*}
    Applying the definition of $\err_{T_0}(a)$ and plugging in for $T_{1.2.1}$ and $T_{1.2.2}$, we get that 
    \begin{equation*}
    \begin{aligned}
        \sqrt{c_n(a)} \| \widehat{\mathbf{P}}_n(a) (\widehat{\theta}_n(a) - \theta_n(a)) \|_{2} &\leq \frac{\sqrt{c_n(a)}}{\sigma_r(\vZ_{T_0}(a))} \left( 34 L^2 U_{T_0}^2 + \frac{108 L^2 \sigma_1(\vZ_{T_0}(a))^2 U_{T_0}^2}{\sigma_r(\vZ_{T_0}(a))^2} + \err_{T_0}(a) \right)^{1/2}\\
        &\leq \frac{\sqrt{c_n(a)}}{\sigma_r(\vZ_{T_0}(a))} \left( \sqrt{34 L^2 U_{T_0}^2} + \sqrt{\frac{108 L^2 \sigma_1(\vZ_{T_0}(a))^2 U_{T_0}^2}{\sigma_r(\vZ_{T_0}(a))^2}} + \sqrt{\err_{T_0}(a)} \right)\\
        &= \frac{\sqrt{34c_n(a)} L}{\wh{\snr}_{T_0}(a)} + \frac{L\sqrt{108c_n(a)} \kappa(\vZ_{T_0}(a))}{\wh{\snr}_{T_0}(a)} + \frac{\sqrt{c_n(a) \err_{T_0}(a)}}{\sigma_r(\vZ_{T_0}(a))}
    \end{aligned}
    \end{equation*}

    \textbf{Bounding $T_2$:} 
    Observe that each $\langle \widehat{\theta}_n(a) - \theta_n(a), \E Y_{n(a), t}^{(a)} \rangle$ is a zero-mean subGaussian random variable with variance at most $\|\theta_n(a)\|_2^2 \sigma^2$. 
    Applying a Hoeffding bound, we can see that $T_2$ is bounded with probability at least $1 - O(A\delta)$, simultaneously for all $a \in \{0, 1, \ldots, A-1\}$, by 
    \begin{equation}
        L \sigma \sqrt{\frac{\log(A / \delta)}{T - T_0}}.
    \end{equation}

    \textbf{Bounding $T_3$:} 
    Similarly to $T_2$, $T_3$ is a zero-mean subGaussian random variable with variance at most $\| \widehat{\theta}_n(a) - \theta_n(a) \|_2^2 \sigma^2$. 
    Therefore we may bound $T_3$, simultaneously for all $a \in \{0, 1, \ldots, A-1\}$, by 
    \begin{equation*}
        \| \widehat{\theta}_n(a) - \theta_n(a) \|_2 \sigma \sqrt{\frac{\log(A / \delta)}{T - T_0}}
    \end{equation*}
    with probability at least $1 - O(A\delta)$. 
    \paragraph{Combining all terms:} 
    By the above decomposition, we have that 
    \begin{equation*}
    \begin{aligned}
        |\widehat{\E}\Bar{Y}_{n,post}^{(a)} - \E \Bar{Y}_{n,post}^{(a)}| &\leq T_{1.1} + T_{1.2} + T_2 + T_3\\
        &\leq \frac{3 \sqrt{c_n(a)}}{\wh \snr_{T_0}(a)} \| \wh \theta_{T_0}(a) - \theta(a) \|_2 + \frac{\sqrt{34 c_n(a)} L}{\wh \snr_{T_0}(a)} + \frac{L\sqrt{108 c_n(a)} \kappa(\vZ_{T_0}(a))}{\wh \snr_{T_0}(a)}\\ &+ \frac{\sqrt{c_n(a) \err_{T_0}(a)}}{\sigma_r(\vZ_{T_0}(a))} + L \sigma \sqrt{\frac{\log(A / \delta)}{T - T_0}} + \sigma \sqrt{\frac{\log(A / \delta)}{T - T_0}} \| \wh \theta_{T_0}(a) - \theta(a) \|_2\\
        &= \frac{\sqrt{c_n(a)}}{\wh \snr_{T_0}(a)} \left( 3 \| \wh \theta_{T_0}(a) - \theta(a) \|_2 + \sqrt{34} L + \sqrt{108}L \kappa(\vZ_{T_0}(a)) \right) + \frac{\sqrt{c_n(a) \err_{T_0}(a)}}{\sigma_r(\vZ_{T_0}(a))}\\ &+ \sigma \sqrt{\frac{\log(A / \delta)}{T - T_0}} \left( \| \wh \theta_{T_0}(a) - \theta(a) \|_2 + L \right)
    \end{aligned}
    \end{equation*}
    Applying the well-balancing assumptions, we get that 
    \begin{equation*}
        |\widehat{\E}\Bar{Y}_{n,post}^{(a)} - \E \Bar{Y}_{n,post}^{(a)}| = \Tilde{O} \left( \frac{r^2 \sqrt{n}}{n \wedge T_0} + \frac{rL\sqrt{n}}{\sqrt{n \wedge T_0}} + \frac{r}{\sqrt{(T - T_0) (n \wedge T_0)}} + \frac{L}{\sqrt{T - T_0}} \right)
    \end{equation*}
    Finally, using the fact that $L = O \left( \frac{r}{\sqrt{n}} \right)$ by Lemma 19 in~\cite{agarwal2023combinations}, we obtain the final result.
    \begin{equation*}
        |\widehat{\E}\Bar{Y}_{n,post}^{(a)} - \E \Bar{Y}_{n,post}^{(a)}| = \Tilde{O} \left( \frac{r^2 \sqrt{n}}{n \wedge T_0} + \frac{r^2}{\sqrt{n \wedge T_0}} + \frac{r}{\sqrt{(T - T_0) (n \wedge T_0)}} \right) 
    \end{equation*}
\end{proof}
\subsection{Proofs for~\Cref{sec:htt}: Learning How to Treat}
\paragraph{Notation} Overloading the notation of~\Cref{sec:setting}, we let 
$\vX_n = (\E Y_{1,pre}, \ldots, \E Y_{n,pre})^{\top}$, 
$\vZ_n = (Y_{1,pre}, \ldots, Y_{n, pre})^{\top}$, 
$\epsilon_{n, pre} = (\pv{\epsilon}{0}_{n,1}, \ldots, \pv{\epsilon}{0}_{n,T_0})^{\top}$, 
$\Epsilon_n = (\epsilon_{1, pre}, \ldots, \epsilon_{n, pre})^{\top}$
$\xi_n = \sum_{t = T_0 + 1}^{\top} \pv{\epsilon}{a_n}_{n,t}$, 
$\Xi_n = (\xi_1, \ldots, \xi_n)^{\top}$, and 
$$\vY_n = \left( \sum_{t=1}^{T_0} \pv{Y}{a_1}_{1,t}, \ldots, \sum_{t=1}^{T_0} \pv{Y}{a_n}_{n,t} \right)^{\top}.$$
Finally, we define quantities such as $\vZ_n(a)$, $\vX_n(a)$, $\vY_n(a)$ analogously to~\Cref{sec:setting}. 
\begin{lemma}[\textbf{Reformulation of average expected post-intervention outcome}]
    Under~\Cref{ass:lfm} and~\Cref{ass:hlsi}, the average expected post-intervention outcome of unit $n$ under intervention $a$ may be written as
    \begin{equation*}
        \E[\pv{\Bar{Y}}{a}_{n,post}] = \frac{1}{T - T_0} \cdot \langle \theta(a), \E[Y_{n,pre}] \rangle,
    \end{equation*}
    for some slope vector $\theta(a) \in \R^{T_0}$.
\end{lemma}
\begin{proof}
    Similar observations have been made in~\cite{harris2022strategyproof, shen2022tale}. 
    For completeness, we include the proof here as well.
    From~\Cref{ass:lfm} and~\Cref{def:apio},
    \begin{equation*}
    \begin{aligned}
        \E[\pv{\Bar{Y}}{a}_{n,post}] = \frac{1}{T - T_0} \cdot \left \langle \sum_{t = T_0 + 1}^T U_t^{(a)}, V_n \right \rangle.
    \end{aligned}
    \end{equation*}
    Applying~\Cref{ass:hlsi}, we see that
    \begin{equation*}
    \begin{aligned}
        \E[\pv{\Bar{Y}}{a}_{n,post}] = \frac{1}{T - T_0} \cdot \left \langle \sum_{t = 1}^{T_0} \theta(a)_t \cdot U_t^{(0)}, V_n \right \rangle
    \end{aligned}
    \end{equation*}
    for some $\theta(a) = [\theta_1(a), \ldots, \theta_{T_0}(a)]^T \in \mathbb{R}^{T_0}$.
\end{proof}
The following lemma follows straightforwardly from translating the notation of~\Cref{thm:emp_bd} to the panel data setting. 
Note that in the panel data setting, $d = T_0$, $\eta = \sigma \sqrt{T - T_0}$, and $\alpha = \sigma^2 (T - T_0)$.
Additionally,~\Cref{ass:noise1} is satisfied with $\gamma = \sigma^2$.
\begin{lemma}\label{lem:l2-panel}
Let $\delta \in (0, 1)$ be an arbitrary confidence parameter. Let $\rho > 0$ be chosen to be sufficiently small, as detailed in Appendix~\ref{app:ridge}. Further, assume that there is some $n_0 \geq 1$ such that $\rank(\vX_{n_0}) = r$ and $\snr_n \geq 2$ for all $n \geq n_0$. Then, with probability at least $1 - O(A\delta)$, simultaneously for all actions $a \in [A]$ and time steps $n \geq n_0$, we have
\begin{align*}
&\left\|\wh{\theta}_n(a) - \theta(a) \right\|_2^2 \leq  \frac{L^2}{\wh{\snr}_n(a)^2}\left[74 + 216\kappa(\vZ_n(a))^2\right] + \frac{2(T-T_0)\err_n(a)}{\sigma_r(\vZ_n(a))^2},
\end{align*}
where $\kappa(\vZ_n(a)) := \frac{\sigma_1(\vZ_n(a))}{\sigma_r(\vZ_n(a))}$, $\|\theta(a)\|_2 \leq L$, and in the above we define the ``error'' term $\err_n(a)$ to be\looseness-1
\begin{equation*}
    \begin{aligned}
        \err_n(a) &:= \frac{32\rho L^2}{T-T_0} +64\sigma^2 \left(\log\left(\frac{A}{\delta}\right) + r\log\left(1 + \frac{\sigma_1(\vZ_n(a))^2}{\rho}\right)\right)\\
        &+ 6\sigma^2\sqrt{2c_n(a)\ell_\delta(c_n(a))} + 10\sigma^2\ell_\delta(c_n(a)) + 6\sigma^2c_n(a).
    \end{aligned}
\end{equation*}
\end{lemma}
\begin{theorem}[\textbf{Prediction error of average expected post-intervention outcome}]
    Let $\delta \in (0, 1)$ be an arbitrary confidence parameter and $\rho > 0$ be chosen to be sufficiently small, as detailed in Appendix~\ref{app:ridge}. 
    Further, assume that Assumptions \ref{ass:lfm} and \ref{ass:hlsi} are satisfied, there is some $n_0 \geq 1$ such that $\rank(\vX_{n_0}) = r$, and $\snr_n(a) \geq 2$ for all $n \geq n_0$. 
    Then with probability at least $1 - O(A\delta)$, simultaneously for all interventions $a \in \{0, \ldots, A-1\}$
    \begin{equation*}
    \begin{aligned}
        |\widehat{\E} \pv{\Bar{Y}}{a}_{n,post} - \E \pv{\Bar{Y}}{a}_{n,post}| &\leq \frac{3 \sqrt{T_0}}{\wh{\snr}_n(a)} \left( \frac{L (\sqrt{74} + 12\sqrt{6}\kappa(\vZ_n(a)))}{(T - T_0) \cdot \wh{\snr}_n(a)} + \frac{\sqrt{2\err_n(a)}}{\sqrt{T-T_0} \cdot \sigma_r(\vZ_n(a))} \right)\\ 
        &+ \frac{2L\sqrt{24T_0}}{(T-T_0) \cdot \wh{\snr}_n(a)} + \frac{12L\kappa(\vZ_n(a))\sqrt{3T_0}}{(T-T_0) \cdot \wh{\snr}_n(a)} + \frac{2\sqrt{\err_n(a)}}{\sqrt{T-T_0} \cdot \sigma_r(\vZ_n(a))}\\
        &+ \frac{L\sigma \sqrt{2\log(A/\delta)}}{T-T_0} + \frac{L\sigma \sqrt{148 \log(A/\delta)}}{\wh{\snr}_n(a) (T-T_0)} + \frac{24\sigma\kappa(\vZ_n(a))\sqrt{6 \log(A/\delta)}}{\wh{\snr}_n(a) (T-T_0)}\\ 
        &+ \frac{2\sigma \sqrt{ \err_n(a) \log(A/\delta)}}{\sigma_r(\vZ_n(a)) \sqrt{T-T_0}}
    \end{aligned}
    \end{equation*}
    where $\widehat{\E}\pv{\Bar{Y}}{a}_{n,post} := \frac{1}{T - T_0} \cdot \langle \wh{\theta}_n(a), Y_{n,pre} \rangle$ is the estimated average post-intervention outcome for unit $n$ under intervention $a$. 
\end{theorem}
\begin{proof}
\begin{equation*}
    \begin{aligned}
        \widehat{\E} \pv{\Bar{Y}}{a}_{n,post} - \E \pv{\Bar{Y}}{a}_{n,post} &:= \frac{1}{T - T_0} (\langle \wh{\theta}_n(a), Y_{n,pre} \rangle - \langle \theta(a), \E Y_{n,pre} \rangle)\\
        &= \frac{1}{T - T_0} ( \underbrace{\langle \wh{\theta}_n(a) - \theta(a), \E Y_{n,pre} \rangle}_{T_1} + \underbrace{\langle \theta(a), \epsilon_{n,pre} \rangle}_{T_2}\\ &+ \underbrace{\langle \wh{\theta}_n(a) - \theta(a), \epsilon_{n,pre} \rangle}_{T_3})\\
    \end{aligned}
\end{equation*}
We begin by bounding $T_1$. 
By assumption we have that $\E Y_{n,pre} \in \spn(\E Y_{1,pre}, \ldots, \E Y_{n-1,pre})$ for all $n \geq n_0$. 
Therefore, 
\begin{equation*}
\begin{aligned}
    \langle \wh{\theta}_n(a) - \theta(a), \E Y_{n,pre} \rangle &= \langle \wh{\theta}_n(a) - \theta(a), \E Y_{n,pre} \vP \rangle\\
    &= \langle \vP \wh{\theta}_n(a) - \theta(a), \E Y_{n,pre} \rangle\\
    &\leq \| \E Y_{n,pre} \|_2 \|(\vP - \wh{\vP}_n(a) + \wh{\vP}_n(a)) (\wh{\theta}_n(a) - \theta(a)) \|_2\\
    &\leq \underbrace{\sqrt{T_0} \cdot \|(\vP - \wh{\vP}_n(a)) (\wh{\theta}_n(a) - \theta(a)) \|_2}_{T_{1.1}} + \underbrace{\sqrt{T_0} \cdot \| \wh{\vP}_n(a) (\wh{\theta}_n(a) - \theta(a)) \|_2}_{T_{1.2}}\\
\end{aligned}
\end{equation*}
By~\Cref{lem:proj} and the second part of~\Cref{lem:int:sv_conc}, term $T_{1.1}$ may be upper-bounded as 
\begin{equation*}
    \begin{aligned}
        T_{1.1} &\leq \sqrt{T_0} \cdot \| \vP - \wh{\vP}_n(a) \|_{op} \cdot \| \wh{\theta}_n(a) - \theta(a) \|_2\\
        &\leq \sqrt{T_0} \cdot \frac{\sqrt{4\beta\left(3\sqrt{n\ell_{\delta/2\cN}(n)} + 5\ell_{\delta/2\cN}(n)\right) + 4n\gamma}}{\sigma_r(\vX_n(a))} \cdot \| \wh{\theta}_n(a) - \theta(a) \|_2\\
        &\leq \sqrt{T_0} \cdot \frac{\sqrt{4\beta\left(3\sqrt{n\ell_{\delta/2\cN}(n)} + 5\ell_{\delta/2\cN}(n)\right) + 4n\gamma}}{ \frac{2}{3}\sigma_r(\vZ_n(a))} \cdot \| \wh{\theta}_n(a) - \theta(a) \|_2\\
        &= \frac{3 \sqrt{T_0}}{\wh{\snr}_n(a)} \cdot \| \wh{\theta}_n(a) - \theta(a) \|_2\\
    \end{aligned}
\end{equation*}
Applying~\Cref{lem:l2-panel}, we see that
\begin{equation*}
    \begin{aligned}
        \frac{1}{T - T_0} T_{1.1} &\leq \frac{3 \sqrt{T_0}}{(T - T_0) \cdot \wh{\snr}_n(a)} \left(\frac{L^2}{\wh{\snr}_n(a)^2}\left[74 + 216\kappa(\vZ_n(a))^2\right] + \frac{2(T-T_0)\err_n(a)}{\sigma_r(\vZ_n(a))^2} \right)^{1/2}\\
        &\leq \frac{3 \sqrt{T_0}}{(T - T_0) \cdot \wh{\snr}_n(a)} \left( \frac{L}{\wh{\snr}_n(a)}\sqrt{74 + 216\kappa(\vZ_n(a))^2} + \frac{\sqrt{2(T-T_0)\err_n(a)}}{\sigma_r(\vZ_n(a))} \right)\\
        &= \frac{3 \sqrt{T_0}}{\wh{\snr}_n(a)} \left( \frac{L (\sqrt{74} + 12\sqrt{6}\kappa(\vZ_n(a)))}{(T - T_0) \cdot \wh{\snr}_n(a)} + \frac{\sqrt{2\err_n(a)}}{\sqrt{T-T_0} \cdot \sigma_r(\vZ_n(a))} \right).
    \end{aligned}
\end{equation*}
Turning our attention to $T_{1.2}$ and using a line of reasoning nearly identical to equations (\ref{eq:proj}), (\ref{eq:T_1}), (\ref{eq:T_2}) in the proof of~\Cref{thm:emp_bd}, we get that with probability at least $1 - \mathcal{O}(A\delta)$,
\begin{equation*}
    \begin{aligned}
        \frac{1}{T-T_0} T_{1,2} &\leq \frac{2 \sqrt{T_0}}{(T-T_0) \cdot \sigma_r(\vZ_n(a))} \left( \left\|\wh{\vZ}_n(a)\wh{\theta}_n(a) - \vX_{n}(a)\theta(a)\right\|_2^2 + \left\|\vX_n(a) \theta(a) - \wh{\vZ}_n(a)\theta(a)\right\|_2^2 \right)^{1/2}\\
        &\leq \frac{2 \sqrt{T_0}}{(T-T_0) \cdot \sigma_r(\vZ_n(a))} \left(32\rho L^2 + 64\sigma^2(T-T_0)\left(\log\left(\frac{A}{\delta}\right) +  r\log\left(1 + \frac{\sigma_1(\vZ_n(a))^2}{\rho}\right)\right) \right.\\ 
        &+ \left. 16L^2U_n^2
        + 6\sigma^2(T-T_0)\sqrt{2c_n(a)\ell_\delta(c_n(a))} + 10\sigma^2(T-T_0)\ell_\delta(c_n(a)) \right. \\
        &+ \left. 6\sigma^2(T-T_0)c_n(a) + \frac{108L^2\sigma_1(\vZ_n(a))^2U_n^2}{\sigma_r(\vZ_n(a))^2} + 18L^2U_n^2 \right)^{1/2}\\
        &= \frac{2 \sqrt{T_0}}{(T-T_0) \cdot \sigma_r(\vZ_n(a))} \left( 34L^2 U_n^2 + \frac{108L^2\sigma_1(\vZ_n(a))^2U_n^2}{\sigma_r(\vZ_n(a))^2} + (T - T_0)\err_n(a) \right)^{1/2}\\
        &\leq \frac{2LU_n\sqrt{24T_0}}{(T-T_0) \cdot \sigma_r(\vZ_n(a))} + \frac{12L\kappa(\vZ_n(a))U_n\sqrt{3T_0}}{(T-T_0) \cdot \sigma_r(\vZ_n(a))} + \frac{2\sqrt{\err_n(a)}}{\sqrt{T-T_0} \cdot \sigma_r(\vZ_n(a))}\\
        &= \frac{2L\sqrt{24T_0}}{(T-T_0) \cdot \wh{\snr}_n(a)} + \frac{12L\kappa(\vZ_n(a))\sqrt{3T_0}}{(T-T_0) \cdot \wh{\snr}_n(a)} + \frac{2\sqrt{\err_n(a)}}{\sqrt{T-T_0} \cdot \sigma_r(\vZ_n(a))}
    \end{aligned}
\end{equation*}
Putting our bounds for $T_{1.1}$ and $T_{1.2}$ together, we get that 
\begin{equation*}
\begin{aligned}
    \frac{\langle \wh{\theta}_n(a) - \theta(a), \E Y_{n,pre} \rangle}{T - T_0} &\leq \frac{3 \sqrt{T_0}}{\wh{\snr}_n(a)} \left( \frac{L (\sqrt{74} + 12\sqrt{6}\kappa(\vZ_n(a)))}{(T - T_0) \cdot \wh{\snr}_n(a)} + \frac{\sqrt{2\err_n(a)}}{\sqrt{T-T_0} \cdot \sigma_r(\vZ_n(a))} \right)\\ 
    &+ \frac{2L\sqrt{24T_0}}{(T-T_0) \cdot \wh{\snr}_n(a)} + \frac{12L\kappa(\vZ_n(a))\sqrt{3T_0}}{(T-T_0) \cdot \wh{\snr}_n(a)} + \frac{2\sqrt{\err_n(a)}}{\sqrt{T-T_0} \cdot \sigma_r(\vZ_n(a))}
\end{aligned}
\end{equation*}
Next we bound $T_2$. 
Note that $\langle \theta(a), \epsilon_{n,pre} \rangle$ is a $\| \theta(a) \|_2 \sigma$-subGaussian random variable.
Therefore via a Hoeffding bound, simultaneously for all actions $a \in [A]$, with probability at least $1 - \mathcal{O}(A \delta)$,
\begin{equation*}
    \frac{\langle \theta(a), \epsilon_{n,pre} \rangle}{T - T_0} \leq L \sigma \frac{\sqrt{2\log(A/\delta)}}{T - T_0}
\end{equation*}
Similarly for $T_3$, $\langle \wh{\theta}_n(a) - \theta(a), \epsilon_{n,pre} \rangle$ is a $\| \wh{\theta}_n(a) - \theta(a) \|_2 \sigma$-subGaussian random variable which, after applying a Hoeffding bound and our bound on $\| \wh{\theta}_n(a) - \theta(a) \|_2$, becomes
\begin{equation*}
\begin{aligned}
    \frac{\langle \wh{\theta}_n(a) - \theta(a), \epsilon_{n,pre} \rangle}{T-T_0} &\leq \sigma \frac{\sqrt{2\log(A/\delta)}}{T - T_0} \left( \frac{L^2}{\wh{\snr}_n(a)^2}\left[74 + 216\kappa(\vZ_n(a))^2\right] + \frac{2(T-T_0)\err_n(a)}{\sigma_r(\vZ_n(a))^2} \right)^{1/2}\\
    &\leq \sigma \frac{\sqrt{2\log(A/\delta)}}{T - T_0} \left( \frac{\sqrt{74}L}{\wh{\snr}_n(a)} + \frac{12\sqrt{6} \kappa(\vZ_n(a))}{\wh{\snr}_n(a)} + \frac{\sqrt{2(T-T_0)\err_n(a)}}{\sigma_r(\vZ_n(a))}\right)\\
    &= \frac{L\sigma \sqrt{148 \log(A/\delta)}}{\wh{\snr}_n(a) (T-T_0)} + \frac{24\sigma\kappa(\vZ_n(a))\sqrt{3 \log(A/\delta)}}{\wh{\snr}_n(a) (T-T_0)} + \frac{2\sigma \sqrt{ \err_n(a) \log(A/\delta)}}{\sigma_r(\vZ_n(a)) \sqrt{T-T_0}}
\end{aligned}
\end{equation*}
Putting everything together, we see that with probability at least $1 - O(A\delta)$,
\begin{equation*}
\begin{aligned}
    |\widehat{\E} \pv{\Bar{Y}}{a}_{n,post} - \E \pv{\Bar{Y}}{a}_{n,post}| &\leq \frac{3 \sqrt{T_0}}{\wh{\snr}_n(a)} \left( \frac{L (\sqrt{74} + 12\sqrt{6}\kappa(\vZ_n(a)))}{(T - T_0) \cdot \wh{\snr}_n(a)} + \frac{\sqrt{2\err_n(a)}}{\sqrt{T-T_0} \cdot \sigma_r(\vZ_n(a))} \right)\\ 
    &+ \frac{2L\sqrt{24T_0}}{(T-T_0) \cdot \wh{\snr}_n(a)} + \frac{12L\kappa(\vZ_n(a))\sqrt{3T_0}}{(T-T_0) \cdot \wh{\snr}_n(a)} + \frac{2\sqrt{\err_n(a)}}{\sqrt{T-T_0} \cdot \sigma_r(\vZ_n(a))}\\
    &+ \frac{L\sigma \sqrt{2\log(A/\delta)}}{T-T_0} + \frac{L\sigma \sqrt{148 \log(A/\delta)}}{\wh{\snr}_n(a) (T-T_0)} + \frac{24\sigma\kappa(\vZ_n(a))\sqrt{6 \log(A/\delta)}}{\wh{\snr}_n(a) (T-T_0)}\\ 
    &+ \frac{2\sigma \sqrt{ \err_n(a) \log(A/\delta)}}{\sigma_r(\vZ_n(a)) \sqrt{T-T_0}}
\end{aligned}
\end{equation*}
Applying~\Cref{ass:spectrum}, the expression simplifies to
\begin{equation*}
\begin{aligned}
    |\widehat{\E} \pv{\Bar{Y}}{a}_{n,post} - \E \pv{\Bar{Y}}{a}_{n,post}| &= \wt{O} \left( \frac{r\sqrt{T_0}}{\sqrt{T_0 \wedge n}} \left( \frac{Lr}{(T - T_0)\sqrt{T_0 \wedge n}} + \frac{r}{\sqrt{(T-T_0)(T_0 \wedge n)}} \right) \right.\\ 
    &+ \frac{Lr\sqrt{T_0}}{(T-T_0)\sqrt{T_0 \wedge n}} + \frac{Lr\sqrt{T_0}}{(T-T_0)\sqrt{T_0 \wedge n}} + \frac{r}{\sqrt{(T-T_0)(T_0 \wedge n)}}\\
    &+ \frac{L}{T-T_0} + \frac{Lr}{(T-T_0)\sqrt{(T_0 \wedge n)}} + \frac{r}{(T-T_0)\sqrt{(T_0 \wedge n)}} + \left. \frac{r}{\sqrt{(T-T_0) (T_0 \wedge n)}}\right).\\
    &\leq \Tilde{O}\left( \frac{Lr\sqrt{T_0}}{(T-T_0)\sqrt{T_0 \wedge n}} + \frac{r^2\sqrt{T_0}}{\sqrt{T-T_0}(T_0 \wedge n)}\right)
\end{aligned}
\end{equation*}
Finally since $L = O\left( \frac{r(T- T_0)}{\sqrt{T_0}} \right)$ (by Lemma 19 in~\cite{agarwal2023combinations}), we get that 
\begin{equation*}
    |\widehat{\E} \pv{\Bar{Y}}{a}_{n,post} - \E \pv{\Bar{Y}}{a}_{n,post}| = \Tilde{O} \left( \frac{r^2}{\sqrt{T_0 \wedge n}} + \frac{r^2 \sqrt{T_0}}{\sqrt{T - T_0} (T_0 \wedge n)} \right)
\end{equation*}
\end{proof}

\end{document}